\documentclass{article}

\usepackage{microtype}

\usepackage[preprint]{neurips_2025}

\usepackage{microtype}
\usepackage{graphicx}
\usepackage{subfigure}
\usepackage{booktabs} %
\usepackage{mathrsfs}
\usepackage{enumitem}
\newlist{inlinelist}{itemize}{1}
\setlist[inlinelist]{label=•, itemjoin={{; }}, itemjoin*={{; and }}, before=\unskip{}, after=\unskip{}, nosep, leftmargin=*, align=parleft}

\usepackage{tikz}
\usetikzlibrary{arrows.meta, positioning}
\usetikzlibrary{shapes.geometric}
\usetikzlibrary{calc}

\usepackage{hyperref}

\usepackage{algorithm}
\usepackage{algpseudocode}

\usepackage{amsmath}
\usepackage{amssymb}
\usepackage{mathtools}
\usepackage{amsthm}
\usepackage{dsfont}
\usepackage{algpseudocode}
\usepackage[capitalize,noabbrev]{cleveref}

\usepackage{subcaption}  %

\theoremstyle{plain}
\newtheorem{theorem}{Theorem}[section]
\newtheorem{proposition}[theorem]{Proposition}
\newtheorem{lemma}[theorem]{Lemma}
\newtheorem{corollary}[theorem]{Corollary}
\theoremstyle{definition}
\newtheorem{definition}[theorem]{Definition}

\theoremstyle{remark}
\newtheorem{remark}[theorem]{Remark}

\usepackage[textsize=tiny]{todonotes}

\newcommand{\R}[0]{\mathbb{R}}

\newcommand{\Prob}[0]{\mathbb{P}}

\newcommand{\eqdef}{\vcentcolon=}

\newcommand\mech[1]{{#1}}
\newcommand\vect[1]{\mathbf{#1}}

\newcommand\p[1]{\left( {#1}\right)}

\newcommand\set[1]{\mathcal{#1}}
\newcommand\proj[0]{\operatorname{Proj}}

\usepackage[utf8]{inputenc} %
\usepackage[T1]{fontenc}    %
\usepackage{hyperref}       %
\usepackage{url}            %
\usepackage{booktabs}       %
\usepackage{amsfonts}       %
\usepackage{nicefrac}       %
\usepackage{microtype}      %
\usepackage{xcolor}         %

\allowdisplaybreaks

\title{Learning with Differentially Private (Sliced) Wasserstein Gradients}

\author{%
  David Rodríguez-Vítores  \\
  Universidad de Valladolid and IMUVA \\
  \texttt{david.rodriguez.vitores@uva.es}\\
  \And
  Clément Lalanne \\
  Institut de Mathématiques de Toulouse\\
  \texttt{clement.lalanne@math.univ-toulouse.fr}\\
  \And
  Jean-Michel Loubes \\
  Université de Toulouse\\
  ANITI \& Regalia INRIA \\
  \texttt{jean-michel.a.loubes@inria.fr}\\
  }

\begin{document}

\maketitle

\begin{abstract}

In this work, we introduce a novel framework for privately optimizing objectives that depend on sliced Wasserstein distances between data-dependent empirical measures. Our main theoretical contribution is a non-trivial analysis of the sensitivity of the Wasserstein gradients to individual data points, derived from an explicit formulation of the gradient in a fully discrete setting. This enables strong privacy guarantees with minimal utility loss.
We demonstrate that standard privacy accounting methods naturally extend to Wasserstein-based objectives, allowing for large-scale private training. This supports a wide range of private machine learning applications involving distribution matching under privacy constraints on the source, the target, or both. These include: (i) an in-processing method for fairness mitigation using a private Wasserstein penalty, and (ii) what we believe is the first approach for training private sliced Wasserstein autoencoders. We validate our framework through experiments showing its ability to effectively balance privacy and utility, offering a theoretically grounded approach to privacy-preserving machine learning with sliced Wasserstein losses.
\end{abstract}

\section{Introduction}
\label{sec:Introduction}
Privacy has emerged as a critical consideration in machine learning, driven by widespread concerns about data misuse and regulatory frameworks such as GDPR and the AI European Act for instance. In fact, traditional machine learning approaches often require access to sensitive information and participate to data leakage, thus raising serious privacy risks.  In fact, analyzing real user data presents new challenges, particularly in terms of privacy. It is well-established that releasing statistics based on such data, without appropriate safeguards can result in serious privacy breaches \cite{narayanan2006break,backstrom2007wherefore,fredrikson2015model,dinur2003revealing,homer2008resolving,loukides2010disclosure,narayanan2008robust,sweeney2000simple,wagner2018technical,sweeney2002k}.

Differential privacy \cite{dwork2006calibrating}, has become the gold standard for addressing these concerns, providing rigorous guarantees on individual data protection during model training. Differential privacy incorporates randomness into the computation process, ensuring that the estimator relies not only on the dataset, but also on an additional source of randomness. This mechanism obscures the influence of individual data points, safeguarding user privacy. The formal definition of differential privacy, as well as the basic results, and tools surrounding it, was moved to \Cref{sec:DifferentialPrivacy}. To the non-expert reader, it is only useful to know that privacy is tuned by two scalar numbers $\varepsilon, \delta \geq 0$, and that the smaller they are, the stronger the privacy guarantees are. Privacy is usually obtained via noise addition mechanisms, of which the amount of noise usually scales with the \emph{sensitivity} of a given query. Prominent organizations like the US Census Bureau~\citep{abowd2018us}, Google~\citep{erlingsson2014rappor}, Apple~\citep{thakurta2017learning}, and Microsoft~\citep{ding2017collecting} have adopted this approach. Notably, an extended body of literature studies the interplay between privacy and learning / statistics \cite{wasserman2010statistical,barber2014privacy,diakonikolas2015differentially,karwa2017finite,bun2019privatehypothesis,bun2021privatehypothesis,kamath2019highdimensional,biswas2020coinpress,kamath2020heavytailed,acharya2021differentially,lalanne:thesis,adenali2021unbounded,cai2021cost,brown2021covariance,cai2021cost,kamath2022improved,lalanne2023private,lalanne2022private,lalanne2024privatedensity,singhal2023polynomial,kamath2023biasvarianceprivacy,kamath2023new}.
 However, achieving  privacy guarantees without significant hampering the efficiency of the models remains challenging, particularly when dealing with complex distributional objectives such as those based on Wasserstein distances. \\
 
Besides, optimal transport, and in particular the Wasserstein distance, has become increasingly important in machine learning due to its ability to capture meaningful geometric relationships between probability distributions, enabling a wide range of applications from generative modeling to domain adaptation and fairness. Given two measures of probability $P$ and $Q$  in $\mathbb R^d$ equipped with its Borel $\sigma$-algebra, the Wasserstein distance $W_p(P,Q)$ is defined as the $p^{th}$ root of the cost associated to the optimal transport plan, i.e., $W_p(P,Q)= \bigl( \inf_{\pi\in \Pi(P,Q)}\  \int_{\mathbb R^d} \|x-y\|_2^p d\pi(x,y) \bigr)^{1/p}$,
where $\Pi(P,Q)$ represents the set of measures of probability in the product space with marginals $P$ and $Q$. Except mentioned otherwise, this article will look at the specific case of $W_2$. Its straightforward geometric interpretation makes it an effective tool for comparing distributions even when the supports do not align, offering a significant advantage over other widely used metrics and divergences.  Thus, it has been successfully applied in a number of areas, including generative models \cite{arjovski2017WGAN}, representation learning \cite{toltstikhin2018WAE}, domain adaptation \cite{courty2017OTDomainAdapt} and fairness \cite{gordaliza2019obtainingFairness,risser2022tacklingAlgorithm,de2024transport,jiang2020wasserstein,chzhen2020fair,gaucher2023fair}. 

To tackle the curse of dimensionality in its computations, two main alternatives have been explored, namely, approximating the OT cost by an entropic regularization, as proposed in~\cite{cuturi2013sinkhorn}, or leveraging the use the Wasserstein distance between one-dimensional projections.

Indeed, denoting by, for any measure of probability $P$ on $\R$, $F_P$ its cumulative distribution function (CDF) 
and $F_P^{-1}$ its quantile function, which is defined as the generalized inverse of $F_P$,
then the $W_2$ Wasserstein distance satisfies $W_2(P, Q) = \| F_P^{-1} - F_Q^{-1}\|_{L^2}
$
for any measures of probability $P, Q$ on $\R$.
This perspective has inspired a variety of distance surrogates that reduce to one-dimensional projections. In this work, we center our attention on the sliced Wasserstein distance \cite{rabin2011wassersteinBarycenter,bonnel2015slicedRadon}, defined as $SW_2(P,Q)= \bigl(  \int_{\mathbb S^{d-1}} W_2^2(\operatorname{Pr}_\vartheta \# P, \operatorname{Pr}_\vartheta  \# Q) d\mu(\vartheta) \bigr)^{1/2}$,
where $\#$ is the \emph{push forward} operation of a measure by a measurable mapping, $\operatorname{Pr}_\vartheta$ the projection along the direction of $\vartheta$, and $\mu$ denotes the uniform measure on the unit sphere $\mathbb S^{d-1}$. Note that the integral on the sphere may be approximated by Monte Carlo methods.
A substantial body of research has demonstrated the effectiveness of the sliced Wasserstein distance as a discrepancy measure for generative modeling \cite{deshpande2018generativelModelling,wu2019slicedGenerativeModels}, representation learning \cite{Kolouri2018SWAE}, domain adaptation \cite{Chen2019slicedWasserteinDiscrepancy} and fairness \cite{risser2022tacklingAlgorithm}. 

In this work, we address the challenge of  %
optimizing objectives which depend on sliced Wasserstein distances between empirical measures.  Assume that we have samples   $\vect{X} = (x_1,\ldots,x_n)\in \mathcal X^n$, $\vect{Z}=(z_1,\dots,z_m)\in \mathcal Z^m$, and denote by $P_{\vect X}$ and $P_{\vect Z}$ the empirical distributions associated with $\vect X$ and $\vect Z$. Assuming that we have a vector of trainable weights $\theta$ that parametrizes two functions $g_\theta : \mathcal{X} \rightarrow \R^d$ and $h_\theta: \mathcal{Z} \rightarrow \R^d$, we want to optimize for  the quantities
\begin{equation}
    \mathscr L(\theta) \eqdef W_2^2 ( g_\theta \# P_{\vect X}, h_\theta \# P_{\vect Z}), \text{ if } d=1, \quad {\rm or} \quad \mathscr L(\theta) \eqdef SW_2^2 ( g_\theta \# P_{\vect X}, h_\theta \# P_{\vect Z})
\end{equation}  for general functions $g_\theta$ and $h_\theta$,
under differential privacy guarantees. Our main theoretical advance lies in carefully analyzing the \emph{sensitivity} of Wasserstein gradients to individual data points, enabled by deriving explicit gradient formulations within a fully discrete setting.

\subsection{Contributions}

The main contribution of this work is to present a novel sensitivity analysis for privately estimating the gradients of $\mathscr L(\cdot)$, which naturally extends to a general framework for private optimization of problems involving $\mathscr L(\cdot)$. 
This general contribution can be split as follows.

    \textbf{1) A non-trivial and tight sensitivity analysis leading to privacy at low cost.}
    Despite $W_2^2$ not enjoying the typical finite sum structure (e.g. $\text{loss} = \frac{1}{n} \sum_{i = 1}^n \ell( g_\theta(x_i))$), we prove that its gradient has a decomposition that is favorable for privacy analysis and that is compatible with standard autodifferentiation frameworks \cite{tensorflow2015-whitepaper,NEURIPS2019_9015,jax2018github}.
    We prove in \Cref{sec:SensitivityPrivacy} that the sensitivity of this gradient (i.e. how much the gradients are allowed to change when changing on individual's data) roughly vanishes as $\frac{1}{n}$ where $n$ is the sample size. The implications of this observation are that it is possible to leverage classical tools in differential privacy to obtain privacy at a vanishing cost in tasks utilizing those gradients.

    \textbf{2) A deep learning framework.}
    As is often the case with differential privacy, the privacy analysis typically assumes that a prescribed set of data-dependent quantities are bounded. In practice, this is often not the case, and one has to resort to the use of clipping \cite{abadi2016deep}, which leads to biases \cite{kamath2023biasvarianceprivacy} in the estimation procedure. Despite the problem not enjoying a finite-sum structure, we show in \Cref{sec:DeepLearningFramework} that similar tricks are applicable to Wasserstein gradients. In addition to that, \Cref{sec:DeepLearningFramework} also shows that privacy accounting \cite{abadi2016deep,dong2019gaussian} is still applicable to Wasserstein gradients, allowing deep learning and scalable applications.

    \textbf{3) Practical applications to distributional learning under privacy.}  
 Our method is, to the best of our knowledge, the first one that enables learning with a distributional loss while providing differential privacy guarantees for both source and target samples, as detailed in Remark \ref{remark:comparison}. Consequently, our framework supports a wide range  of novel privacy-preserving learning tasks involving distributional matching. We illustrate its relevance through two key applications:
\begin{inlinelist}
    \item \textbf{Private sliced Wasserstein autoencoders. (see Section \ref{section:swae})}  
    We propose the first differentially private procedure for training  sliced Wasserstein autoencoders, enabling scalable and privacy-preserving representation learning, with natural applications to generative modeling.
    \item \textbf{Fairness via private in-processing. (see Section \ref{section:fairness})}  
    Our framework enables the private optimization of sliced Wasserstein distances, including between conditional distributions based on sensitive attributes. This offers a novel private in-processing approach to mitigate bias of Machine Learning algorithms. This privacy-preserving fairness regularization strategy constitutes, to our knowledge, a new contribution to the literature.

\end{inlinelist}

\subsection{Related Work}

Our work directly compares to the work of \cite{rakotomamonjy2021DPslicedWasserstein}, which extends the ideas from \cite{harder2021dp}—originally applied to the Maximum Mean Discrepancy (MMD)—to the sliced Wasserstein loss. Their work establishes privacy guarantees for the value of the sliced Wasserstein distance, based on the use of the post-processing lemma. However, their privacy guarantees are insufficient for training models privately, except in simple scenarios such as the generative model proposed in \cite{harder2021dp}. Indeed, their technique requires private data to be static (as opposed to trainable), which prevents the use of this technique when the Wasserstein loss appears downstream in the learning process. In contrast, our work is significantly broader in scope, and adapts to a wider range of problems, as discussed in \cref{remark:comparison}. A more detailed comparison to \cite{rakotomamonjy2021DPslicedWasserstein} is provided later in the article, both in terms of applicability of the methods and numerical results when both are applicable. \cite{Liu_Yu} follows the same line of work as \cite{rakotomamonjy2021DPslicedWasserstein}, extending their methodology to an alternative definition of the sliced Wasserstein distance.\\

In a different vein, other existing works develop task-specific private methodologies leveraging optimal transport. The sliced Wasserstein distance has been applied in data generation by \cite{segag2023gradientFlow} from a different approach based on gradient flows. \cite{tien2019DPOTdomainAdapt} tackled differentially private domain adaptation with optimal transport by perturbing the optimal coupling between noisy data. Recently, \cite{xian2024DPfairRegression} proposed a post-processing method based on the Wasserstein barycenter of private histogram estimators of conditional densities to obtain a fair and private regressor. The private estimation of optimal transport maps was recently studied in \cite{lalanne:hal-04923578}. Beyond these approaches, optimal transport has also been explored in novel privacy paradigms unrelated to our work \cite{pierquin2024Pufferfish,Kawamoto2019localObfuscation,yang2024wassersteinDP}.

Due to space constraints, the extended related work was moved to \Cref{sec:ExtendedRelatedWork}.

\section{Differential Privacy}
\label{sec:DifferentialPrivacy}

Differential privacy \cite{dwork2006calibrating} starts with fixing a dataset space $\set{D}$, the space in which we expect the dataset to live, and a neighboring relation $\sim$ on $\set{D}$. For $\vect{D}, \vect{\Tilde D} \in \set{D}$, we write $\vect{D} \sim \vect{\Tilde D}$ when $\vect{D}$ and $\vect{\Tilde D}$ are \emph{neighbors} (see below). Differential privacy then imposes that a \emph{randomized} mechanism (i.e. a conditional kernel of probabilities) $\mech{M} : \set{D} \rightarrow \set{O}$ makes 
$\mech{M}(\vect{D})$ hard to discriminate (in a statistical sense) from $\mech{M}(\vect{\Tilde D})$ for any pair of neighbors $\vect{D} \sim \vect{\Tilde D}$.

The neighboring relation $\sim$ is \emph{application specific} and is usually either the \emph{addition / deletion} relation ($\vect{D}$ and $\vect{\Tilde D}$ are neighbors iff one can be obtained from the other by adding or removing the data of one individual from the dataset) or the \emph{replacement} relation ($\vect{D}$ and $\vect{\Tilde D}$ are neighbors iff one can be obtained from the other by changing the data of one individual from either dataset).  In general, it is useful to the reader to understand
$\vect{D} \sim \vect{\Tilde D}$ as : ``The difference between $\vect D$ and $\vect{\Tilde D}$ only comes
from one individual’s data”.

In our paper, due to the splitting of the data into separate categories in the Wasserstein distance, and because of potential asymmetry that may arise in their  treatments, we will occasionally employ modified definitions of neighboring relations, which can be encompassed within the following family, indexed by the number of classes $k\geq 1$. Note that the case $k=1$ reduces to the usual \textit{replacement} relation. 

\begin{definition}{(k-end neighboring relation)} \label{definition:neighboring} Let $\mathcal D  = \mathcal D_1^{n_1} \times \ldots \times \mathcal D_k^{n_k}$ be the set of partitioned datasets with sizes $n_1,\ldots,n_k \geq 1$. Given two datasets $\vect D = (\vect{D^1},\ldots,\vect{D^k})$, $\vect{\Tilde D} =  (\vect{\Tilde D^1},\ldots,\vect{\Tilde D^k}) \in \mathcal D$, we say that $\vect D \sim_k \vect{\Tilde D}$
if there exist and index $j\in[k] \eqdef \{1,\ldots,k\}$ such that $\vect{D^i}= (d^i_1,\ldots,d_{n_i}^i)$ and $\vect{\Tilde D^i}= (\Tilde d^i_1,\ldots,\Tilde d_{n_i}^i)$ coincide up to a
permutation of the elements if $i\neq j$, and up to a permutation and a replacement of one of
the $d_l^i$'s by any element in $\mathcal D_i$ if $i=j$.
\end{definition}

The historic definition of differential privacy \cite{dwork2006calibrating,dwork2006differentialPrivacy,dwork2006our} with $(\varepsilon,\delta)$ reads:

\begin{definition}[{$(\varepsilon, \delta)$-DP \cite{dwork2006our}}]
\label{definitionConcentratedDifferentialPrivacy}
    A randomized mechanism $\mech{M} : \set{D} \rightarrow \set{O}$ is $(\varepsilon, \delta)$-differentially private ($(\varepsilon, \delta)$-DP) if $\forall \ \vect{D}\sim \vect{\Tilde D}$, and $\forall$ measurable $S \subset \set{O}$,
\begin{equation*}
 \Prob \p{M(\vect{D}) \in S} \leq e^{\varepsilon} \Prob \p{M(\vect{\Tilde D}) \in S} + \delta \;,
\end{equation*}
where the randomness is taken on $M$ only.
\end{definition}
A ubiquitous building block for building private mechanisms is the so-called \emph{Gaussian mechanism} which consists in adding independent Gaussian noise to the output of a deterministic mapping. Quantifying the privacy of this (now randomized) mechanism then boils down to controlling the variations of the deterministic mapping on neighboring datasets, as formalized in the following lemma.
\begin{lemma}[Privacy of the Gaussian mechanism (Corollary of Theorem 2.7, Corollary 3.3 and Corollary 2.13 in \cite{dong2019gaussian})]
\label{factProvacyGaussianMechanism}
    Given a deterministic function $h$ mapping a dataset to a quantity in $\R^{d}$, one can define the $l_2$-sensitivity of $h$ as
    \begin{equation*}
\begin{aligned}
    \Delta_2h \eqdef \sup_{\vect{D}\sim \vect{\Tilde D}} \big\| h (\vect{D})  - h (\vect{\Tilde D})\big\|_2 \;.
\end{aligned}
\end{equation*}
When this quantity is finite, for any $\sigma > 0$, the Gaussian mechanism defined as 
$
        \vect{D} \mapsto h(\vect{D}) + \sigma \mathcal{N}(0, I_{d}) \;,
    $
is $(\varepsilon, \delta(\varepsilon))$-DP for any $\varepsilon \geq 0$ where, by noting $\mu = \frac{\Delta_2 h}{\sigma}$, $    \delta(\varepsilon) = \Phi \bigl( - \frac{\varepsilon}{\mu} + \frac{\mu}{2}\bigr) - e^{\varepsilon} \Phi \bigl( - \frac{\varepsilon}{\mu} - \frac{\mu}{2}\bigr)$, where $\Phi$ denotes the standard normal CDF.
\end{lemma}
In particular, \Cref{sec:SensitivityPrivacy} proves that the Wasserstein gradients enjoy such property, motivating the methods presented in this article.

    In addition, private mechanisms enjoy several important properties. They are stable under composition, meaning the privacy loss from multiple sequential accesses to a dataset can be quantitatively bounded \cite{dwork2014algorithmic,dong2019gaussian}. Privacy is also amplified by subsampling, where applying a private mechanism to a random subset of the data leads to stronger privacy guarantees \cite{DBLP:journals/corr/abs-2210-00597}. Finally, private mechanisms are stable under post-processing, ensuring that any data-independent transformation of their output does not degrade privacy \cite{dwork2014algorithmic,dong2019gaussian}.

\section{Wasserstein Gradients}
\label{sec:WassersteinGradients}

Obtaining appropriate privacy guarantees involves deriving a concise and tractable closed-form expression for the squared Wasserstein distance between one-dimensional empirical distributions.

In the following, given sample of observations $x_i \in \mathbb{R}$ for $i\in[n]$, %
we denote its order statistics by 
$ x_{(1)} \leq x_{(2)} \leq \cdots \leq x_{(n)}.$
Given two discrete probabilities on the real line $P_\vect{U} = \frac{1}{n}\sum_{i=1}^n \delta_{U_i}$ and $P_\vect{V}=\frac{1}{m}\sum_{j=1}^m \delta_{V_j}$, using the characterization of $W_2^2(P_{\vect U},P_{\vect V})$ in terms of quantile functions, it follows that if we define the weights 
$
    R_{i,j}= \lambda \bigl(\bigl(\frac{i-1}{n},\frac{i}{n}\bigr]\cap\bigl(\frac{j-1}{m},\frac{j}{m}\bigr]\bigr), \ i\in [n],j\in[m],
$
where $\lambda$ denotes the Lebesgue measure, and consider the rank permutations $\sigma,\tau$ such that $U_i = U_{(\sigma(i))}$ for each $i \in [n]$ and $V_j = V_{(\tau(j))}$ for each $j\in[m]$, then 
\begin{proposition}\label{lemma:expression_W2} With the above notation,
   $
   W_2^2(P_{\vect U},P_{\vect V})  = \sum_{i=1}^{n} \sum_{j=1}^{m} R_{\sigma(i),\tau(j)}(U_{i}-V_{j} )^2 \;.
$
\end{proposition}
Proposition~\ref{lemma:expression_W2}  expresses the Wasserstein distance  as a weighted sum of squared differences, with weights $R_{\sigma(i),\tau(j)}$, that depend only on the rank permutations. Thus, the partial derivative with respect to $U_i$,  is well defined, as long as its rank $\sigma(i)$ remains unchanged in a neighborhood of $U_i$. 
\begin{proposition}
    \label{corollary:FormulaGradientsGeneral}
    With all the previous definitions, 
    $W_2^2(P_{\vect U},P_{\vect V})$ is differentiable as a function of $\vect U=(U_1,\ldots,U_n)$ in the set of points verifying $U_{(1)}<\ldots<U_{(n)}$, and its gradient is given by
    $\label{eq:gradient_U}
        \nabla_{U}W_2^2(P_{\vect U},P_{\vect V}) = \bigl( 2\sum_{j=1}^m R_{\sigma(i),\tau(j)} (U_i-V_j)\bigr)_{i\in[n]} \;.
    $
    Similarly, as a function of $\vect V=(V_1,\ldots,V_m)$, $W_2^2(P_{\vect U},P_{\vect V})$  is differentiable in the set of points verifying $V_{(1)}<\ldots<V_{(m)}$, and 
    $\label{eq:gradient_V}
        \nabla_{V}W_2^2(P_{\vect U},P_{\vect V}) = \bigl( 2\sum_{i=1}^n R_{\sigma(i),\tau(j)} (V_j-U_i)\bigr)_{j\in[m]} \;.
    $
\end{proposition}
This result offers a straightforward alternative to the empirical approximation of the Wasserstein gradient between absolutely continuous measures presented in \cite{risser2022tacklingAlgorithm}, and generalizes the gradient formula used in \cite{bonnel2015slicedRadon} and \cite{tanguy2023propertiesDiscreteSliced} to distributions with different sample sizes.  From a practical perspective, the lack of differentiability when some of the points coincide is not a significant concern. In such rare cases, the rank permutation is not unique. Selecting one of these permutations, the formulas of Proposition \ref{corollary:FormulaGradientsGeneral} allow us to compute (incorrect) gradients, take a step, and continue. It should be noted that this approach has been implicitly assumed in previous papers \cite{deshpande2018generativelModelling, Kolouri2018SWAE} relying on automatic differentiation with satisfactory empirical results.
With a slight abuse of notation, we will use the term \textit{gradient} in the following sections to denote the values in Proposition \ref{corollary:FormulaGradientsGeneral}. Even outside the set of differentiability points, we will still be able to obtain privacy guarantees.

\section{A Private Surrogate for Wasserstein Gradients}\label{sec:SensitivityPrivacy}

This section presents our main theoretical contribution (see \Cref{theorem:gradient_sensitivity}), which is an upper bound on the sensitivity of Wasserstein gradients. This bound directly enables the application of the Gaussian mechanism to ensure differential privacy (see \Cref{factProvacyGaussianMechanism}).

Considering samples   $\vect{X} = (x_1,\ldots,x_n)\in \mathcal X^n$ and $\vect{Z}=(z_1,\dots,z_m)\in \mathcal Z^m$, %
and denoting by $U_i=g_\theta(x_i)$ and $V_j = h_\theta(z_j)$ the activations on which one wishes to apply Wasserstein constraints (note that it matches the notations of \Cref{sec:WassersteinGradients}), we look at the sensitivity of $\nabla_\theta W_2^2 \eqdef \nabla_\theta W_2^2 ( g_\theta \# P_{\vect X}, h_\theta \# P_{\vect Z})$. For clarity of presentation and due to its importance in various applications, we present the  analysis for the one-dimensional Wasserstein distance, assuming $ g_\theta(x), h_\theta(z) \in \mathbb{R}$. The extension to the sliced case is simple, as explained in \cref{remark:extension_sliced}. 

\Cref{corollary:FormulaGradientsGeneral} and the chain rule under suitable assumptions give the following expression,
\begin{equation*}
\begin{aligned}
        \nabla_{\theta}W_2^2  
        &= 2 \sum_{i=1}^n \sum_{j=1}^m R_{\sigma(i),\tau(j)} (g_{\theta}(x_i)-h_{\theta}(z_j)) \nabla_{\theta} g_{\theta}(x_i) \\
        &\qquad\qquad+ 2 \sum_{i=1}^n \sum_{j=1}^m R_{\sigma(i),\tau(j)} (h_{\theta}(z_j)-g_{\theta}(x_i)) \nabla_{\theta} h_{\theta}(z_j)\;.
\end{aligned}
\end{equation*}
This decomposition enables the following sensitivity analysis :

\begin{theorem}\label{theorem:gradient_sensitivity}
    With all the previous notation, assuming that there exists constants $M,L_1,L_2\geq 0$ such that for each $\theta\in\Theta$, $x\in\mathcal X$  and $z\in\mathcal Z$,
    \begin{enumerate}
        \item[(i)] \label{assumption1}$|g_\theta(x)|\leq M$, $|h_\theta(z)|\leq M \ .$
        \item[(ii)] \label{assumption2}$\|\nabla_\theta g_\theta(x)\|_2\leq L_1$, $\|\nabla_\theta h_\theta(z)\|_2\leq L_2 \ ,$
    \end{enumerate}  
    then 
    \begin{itemize}
        \item[(a)]  Under the neighboring relation $\sim_1$ in $\mathcal D = \mathcal X^n$, if we define $\Phi_\theta(\vect X)= \nabla_{\theta} W_2^2(g_\theta\# P_{\vect X}, h_\theta\# P_{\vect Z})$, then
$ \Delta_2\Phi_\theta   \leq  4M\frac{3L_1+L_2}{n} \ .$
\item[(b)] Under the neighboring relation $\sim_2$ in $\mathcal D = \mathcal X^n\times \mathcal Z^m$, if we define $\Psi_\theta(\vect X, \vect Z)= \nabla_{\theta} W_2^2(g_\theta\# P_{\vect X}, h_\theta\# P_{\vect Z})$, then
$ \Delta_2\Psi_\theta   \leq  4M\max\bigl\{ \frac{3L_1+L_2}{n}  , \frac{L_1+3L_2}{m}  \bigr\}\ .$

    \end{itemize}

\end{theorem}

\begin{remark}[\textbf{Extension to arbitrary  dimension}]
\label{remark:extension_sliced}
\cref{theorem:gradient_sensitivity_sliced} in Appendix \ref{append:slice} extends Theorem~\ref{theorem:gradient_sensitivity} to the multidimensional setting $g_\theta(x),h_\theta(z)\in \mathbb R^d$, using the sliced Wasserstein distance $SW_{2}^2$ or its Monte Carlo approximation. In any case, \cref{theorem:gradient_sensitivity_sliced} follows directly from \cref{theorem:gradient_sensitivity} and the fact that the sliced gradient is an average (in the form of an integral of a sum) of the gradients $\nabla_{\theta} W_2^2( (\operatorname{Pr}_\vartheta \circ g_\theta)\# P_{\vect X}, (\operatorname{Pr}_\vartheta \circ h_\theta)\# P_{\vect Z})$.  The conclusions of \cref{theorem:gradient_sensitivity_sliced} remain identical to  \cref{theorem:gradient_sensitivity}, but now require uniform control over $\vartheta$ of the bounds in Assumptions (i) and (ii) for the composite functions $\operatorname{Pr}_\vartheta \circ g_\theta, \operatorname{Pr}_\vartheta \circ h_\theta$. To achieve this, Assumption (i) is replaced by $\|g_\theta(x)\|_2\leq M$, $\|h_\theta(x)\|_2\leq M$ and Assumption (ii) by $\|\mathcal J_\theta g_\theta(x)\|_2\leq L_1$, $\|\mathcal J_\theta h_\theta(z)\|_2\leq L_2$, where $\|\cdot \|_2$ denotes here the spectral norm of a matrix. Note that, if $d=1$,  $\mathcal J_\theta g_\theta = \nabla_\theta g_\theta$ and the spectral norm coincides with the euclidean norm.
\end{remark}

\begin{remark}[\textbf{About the assumptions}]
Assumption (i) is a uniform boundedness condition. Assumption (ii) is verified as soon as  $g_\theta$ and $h_\theta$ are Lipschitz with respect to the parameter $\theta$. In practice, for general models, such boundedness assumptions may be difficult to satisfy. In such cases, it is possible to resort to \emph{clipping} techniques which consist of artificially enforcing such constraints by projections. The sensitivity analysis then holds independently of the model. For more information, see \Cref{sec:DeepLearningFramework}.
\end{remark}

\begin{remark}[\textbf{General optimizers and privacy accounting}]
\Cref{theorem:gradient_sensitivity} together with \Cref{factProvacyGaussianMechanism} enable one to privately estimate one Wasserstein gradient at a point. When used as a substitution for the usual gradients in any first-order optimizer (e.g. SGD), it is possible to quantify the privacy of the whole optimization procedure with \emph{privacy accounting}. Such considerations as well as subsampling are discussed in \Cref{sec:DeepLearningFramework}.
\end{remark}

\begin{remark}[\textbf{Computational complexity}]
\label{remark:computations}
An important limitation of this approach is that it requires access to the Jacobian of the layer where the sliced Wasserstein constraint is applied, which can introduce additional computational overhead.
A broader discussion on this topic is provided in \Cref{sec:computational_details}.
\end{remark}

\begin{remark}[\textbf{Comparison to the baseline} \cite{rakotomamonjy2021DPslicedWasserstein}]\label{remark:comparison}
Assuming $h_\theta = I_d$, by the chain rule, and with a slight abuse of notation,
\begin{equation*} \nabla_\theta SW_2^2(g_\theta(\vect X), \vect{Z} ) = {\nabla_{g_\theta(\vect X)} SW_2^2(g_\theta(\vect X), \vect{Z} )} \times \mathcal J_\theta g_\theta(\vect X)\ .\end{equation*} The approach in \cite{rakotomamonjy2021DPslicedWasserstein} provides privacy guarantees only for the first term in the decomposition. Therefore, privacy guarantees can only be derived in cases where the trained function $g_\theta$ is not directly applied to private data, allowing the second term to be ignored. In other words, privacy guarantees can only be given with respect to fixed $\vect Z$. As a result, the training procedure of \cite{rakotomamonjy2021DPslicedWasserstein} is considerably more limited in scope and primarily suited for simple tasks such as data generation. A numerical comparison of both methods when applicable is provided  in \Cref{section:distribution_matching}. 
Additionally, \Cref{section:swae} presents a comparison in the context of data generation, leveraging private sliced Wasserstein autoencoders as a byproduct of our framework.
\end{remark}

\section{Private Sliced Wasserstein Autoencoders}\label{section:swae}

\begin{figure}[ht]
    \centering
    \begin{minipage}[t]{0.65\textwidth}
        \vspace{0pt}
        \centering
        \resizebox{\linewidth}{!}{
        \begin{tikzpicture}[
            node distance=1.5cm and 1.2cm,
            box/.style={draw, minimum width=1.5cm, minimum height=1cm, align=center, fill=blue!10},
            roundbox/.style={draw, shape=ellipse, minimum height=1cm, minimum width=2cm, align=center, fill=orange!15},
            latent/.style={draw, shape=ellipse, minimum height=1cm, minimum width=2cm, align=center, fill=green!15},
            arrow/.style={-{Latex[length=3mm]}, thick},
            ]

        \node[roundbox] (input) {$\vect{x}_i$};
        \node[box, right=of input] (encoder) {Encoder\\$\varphi_\theta$};
        \node[latent, right=of encoder] (latent) {$\varphi_\theta(x_i)$};
        \node[box, right=of latent] (decoder) {Decoder\\$\psi_\theta$};
        \node[roundbox, right=of decoder] (output) {$\hat{\vect{x}}_i$};

        \node[below=1.5cm of latent] (distmatch) {\footnotesize $\varphi_\theta \# P_{\vect X} \overset{SW_2^2}{\longleftrightarrow} P_{\vect Z}$};
        \node[roundbox, below=1.5cm of output] (reconloss) {\footnotesize $\ell(\hat{\vect{x}}_i, \vect{x}_i)$};

        \draw[arrow] (input) -- (encoder);
        \draw[arrow] (encoder) -- (latent);
        \draw[arrow] (latent) -- (decoder);
        \draw[arrow] (decoder) -- (output);

        \draw[arrow, dashed] (latent) -- (distmatch);
        \draw[arrow, dashed] (output) -- (reconloss);
        \draw[arrow, dashed] (input.south) -- ++(0,-1.2) -| (reconloss.west);

        \node[below=0.3cm of distmatch] {\footnotesize Sliced Wasserstein regularization};
        \node[below=0.3cm of reconloss] {\footnotesize Reconstruction loss};

        \node[below=3.0cm of latent, align=center] (loss) {\footnotesize Total loss: $\mathscr{L}_\alpha = (1-\alpha)\cdot\frac{1}{n}\sum_i \ell(\hat{\vect{x}}_i, \vect{x}_i) + \alpha \cdot SW_2^2(\varphi_\theta \# P_{\vect X}, P_{\vect Z})$};

        \end{tikzpicture}
        }
    \end{minipage}%
    \hfill
    \begin{minipage}[t]{0.33\textwidth}
        \vspace{0pt}
        \caption{Sliced Wasserstein autoencoder with privacy-aware training. The encoder maps input $\vect{x}_i$ to a latent code $\varphi_\theta(x_i)$, encouraged to match the prior $P_{\vect Z}$. The decoder reconstructs $\hat{\vect{x}}_i$. The loss combines reconstruction and sliced Wasserstein distance in latent space.}
        \label{fig:SWAE_diagram}
    \end{minipage}
\end{figure}

Building on our previous analysis, we present what is, to the best of our knowledge, the first differentially private training procedure for the sliced Wasserstein autoencoders introduced in \cite{Kolouri2018SWAE}. 

\subsection{General method}
Given a private dataset $\vect{X} \in \mathcal{X}^n$, the goal is to learn, in a privacy-preserving manner, an encoder $\varphi_{\theta_a}$ and a decoder $\psi_{\theta_b}$ by minimizing a reconstruction loss $\ell$, regularized with a sliced Wasserstein penalty in the latent space. This term encourages the encoded representations to match a non-private random sample $\vect{Z}$ drawn from a prescribed distribution $Q$ over the latent space $\mathbb{R}^{d_0}$. Denoting $\theta = (\theta_a, \theta_b)$, $\varphi_{\theta} = \varphi_{\theta_a}$, and $\psi_{\theta} = \psi_{\theta_b}$, the objective is to minimize
\begin{align}\label{eq:loss_autoencoder_kolouri}
    \mathscr L_\alpha(\theta) = (1-\alpha)\frac{1}{n}\sum_{i=1}^n \ell(\psi_\theta (\varphi_\theta(x_i)),x_i) +\alpha SW_2^2\bigl(\varphi_\theta \# P_{\vect X} , P_{\vect Z}\bigr)  \ .
\end{align}
where $\alpha \in[0,1]$ is the regularization parameter. To train the model with privacy guarantees, one must control the sensitivity of the gradient updates. For the first term in Equation~\ref{eq:loss_autoencoder_kolouri}, the finite-sum structure of the loss function carries over to the gradient. As long as $\|\nabla_\theta \ell(\psi_\theta (\varphi_\theta(x_i)), x_i)\|_2 \leq C$ for all $x_i \in \mathcal{X}$, the sensitivity is bounded by $2C/n$ under the substitution relation $\sim_1$. Therefore, if $\varphi_\theta$ satisfies Assumptions (i) and (ii) in Theorem \ref{theorem:gradient_sensitivity_sliced}, 
$%
    \Delta_2 \mathscr L_\alpha (\theta) \leq (1-\alpha)\frac{2C}{n} + \alpha \cdot \frac{12M  L}{n} \ .
$

This bound allows us to benefit from large sample sizes, enabling private gradient computation with minimal noise.  Building on the ideas presented in Appendix~\ref{sec:DeepLearningFramework}, the training procedure also incorporates:
\begin{inlinelist}
\item \textbf{Clipping:} If the assumptions are not satisfied, we resort to clipping the individual gradients $\nabla_\theta \ell(\psi_\theta (\varphi_\theta(x_i)), x_i)$ by a constant $C > 0$, and apply inner clipping to the sliced Wasserstein gradients using constants $M>0$ for the output of the encoder and $L > 0$ for the spectral norm of the Jacobian matrix, as detailed in \Cref{remark:extension_sliced};
\item \textbf{Subsampling:} The gradient is estimated over batches of size $n'$. The sensitivity of this estimate is computed by replacing $n$ with $n'$ on the above expression.
\end{inlinelist}

\begin{figure}[ht]
    \centering
    \begin{minipage}[t]{0.73\textwidth}
        \vspace{0pt} %
        \centering
        \includegraphics[width=0.98\linewidth]{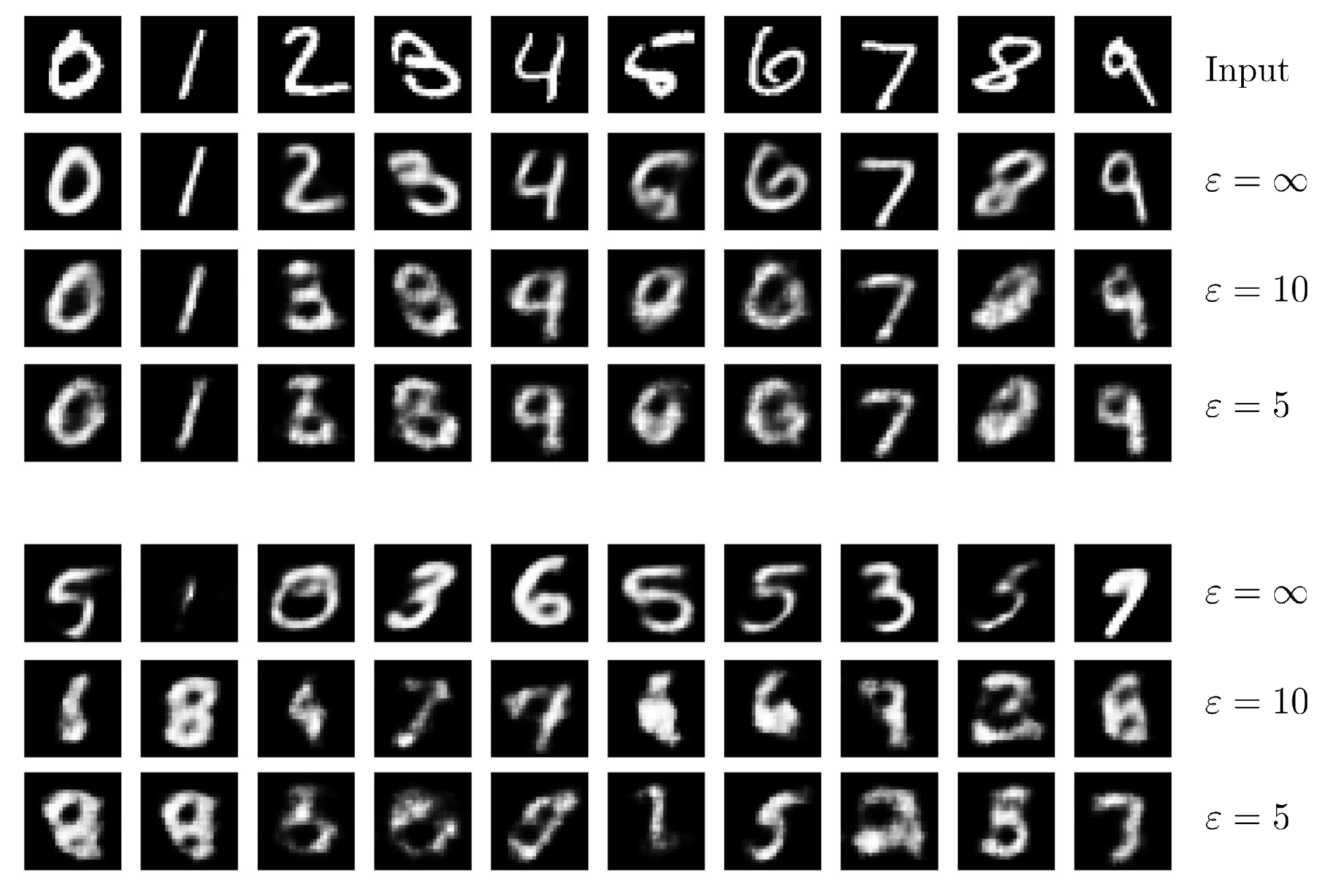}
    \end{minipage}%
    \hfill
    \begin{minipage}[t]{0.26\textwidth}
        \vspace{0.5cm} %
        \caption{Benchmarking the capabilities of a $(\varepsilon,\delta)$-DP sliced Wasserstein autoencoder on MNIST with $\delta = 10^{-5}$ and varying $\varepsilon$. 
        \textbf{Top:} Reconstructed digits from the test dataset. The first row shows the earliest sample of each digit (0–9) in the test set, followed by reconstructions. \\
        \textbf{Bottom:} Generated samples from the same model by decoding noise from the latent space.}
        \label{fig:mnist_autoencoder_results}
    \end{minipage}
\end{figure}

\subsection{Experimental details}

In our experiments, we considered $\vect{X} \in (\mathbb{R}^{28 \times 28})^{60000}$ to be the \textit{MNIST} or \textit{Fashion-MNIST} (training) dataset, and $\vect{Z}$ to be a random sample of the same length drawn from the uniform distribution over the unit ball in $\mathbb{R}^{d_0}$, with $d_0 = 6$. The encoder and decoder are defined as neural networks, with the architectures described in Table~\ref{tab:autoencoder_architecture} in Appendix \ref{appendix:SWAE}. We trained both models using  binary cross-entropy as the reconstruction loss, and regularization parameter $\alpha = 0.1$. Private gradients were computed over random batches of size $n' = 600$, and parameter updates followed the ADAM rule (\cite{bengio2015adam}). The clipping values used were $M = 1.5$ for the decoder output, $L = \sqrt{6}$ for clipping the Jacobian matrix of the encoder’s last layer (following the naive approach of Remark \ref{remark:degraded_bound}), and $C = 1$ for clipping the individual gradients of the reconstruction loss. The number of training iterations was fixed at 5000, and we added the required amount of noise to each gradient to achieve $(\varepsilon, \delta)$-DP after all iterations. Finally, the Monte Carlo approximation of $SW_2$ was computed using $100$ randomly sampled projection directions. 
The reconstructions and latent space embeddings for \textit{MNIST} are presented in \Cref{fig:mnist_autoencoder_results} and \Cref{fig:mnist_encoding}, respectively. We can see that the distributional constraints on the embeddings, including the shape of their support, are correctly enforced while still allowing meaningful reconstructions. The method presented has, to the best of our knowledge, no alternative in the context of differential privacy.

\subsection{Generative modeling as a byproduct - comparison to the baseline}

 A byproduct of sliced Wasserstein autoencoders is their ability to generate data that resemble the distribution of the training set $\vect{X}$. Indeed, since the original dataset is matched to a pre-specified distribution, decoding new samples from this distribution produces synthetic data. Since the entire training procedure respects differential privacy, the result is a privacy-preserving data generator. Examples of generated data for the MNIST dataset are featured in \Cref{fig:mnist_autoencoder_results}. Our results are competitive with existing approaches—see, for instance, Figures 2, 3, and 4 in \cite{segag2023gradientFlow}, where their private sliced Wasserstein flow approach is compared with the private data generation method of \cite{rakotomamonjy2021DPslicedWasserstein}. 
 
 Finally, similar results for the \textit{Fashion-MNIST} dataset are provided in Appendix \ref{appendix:SWAE}.

\begin{figure}[ht]
    \centering
    \begin{minipage}[t]{0.15\textwidth}
        \vspace{0.4cm}
        \captionof{figure}{Encoded latent space MNIST samples for the autoencoder, trained under $(\varepsilon,\delta)$-DP, with $\delta = 10^{-5}$ and varying values of $\varepsilon$.}
        \label{fig:mnist_encoding}
    \end{minipage}%
    \hfill
    \begin{minipage}[t]{0.82\textwidth}
        \vspace{0pt}
        \centering
        \includegraphics[width=\linewidth]{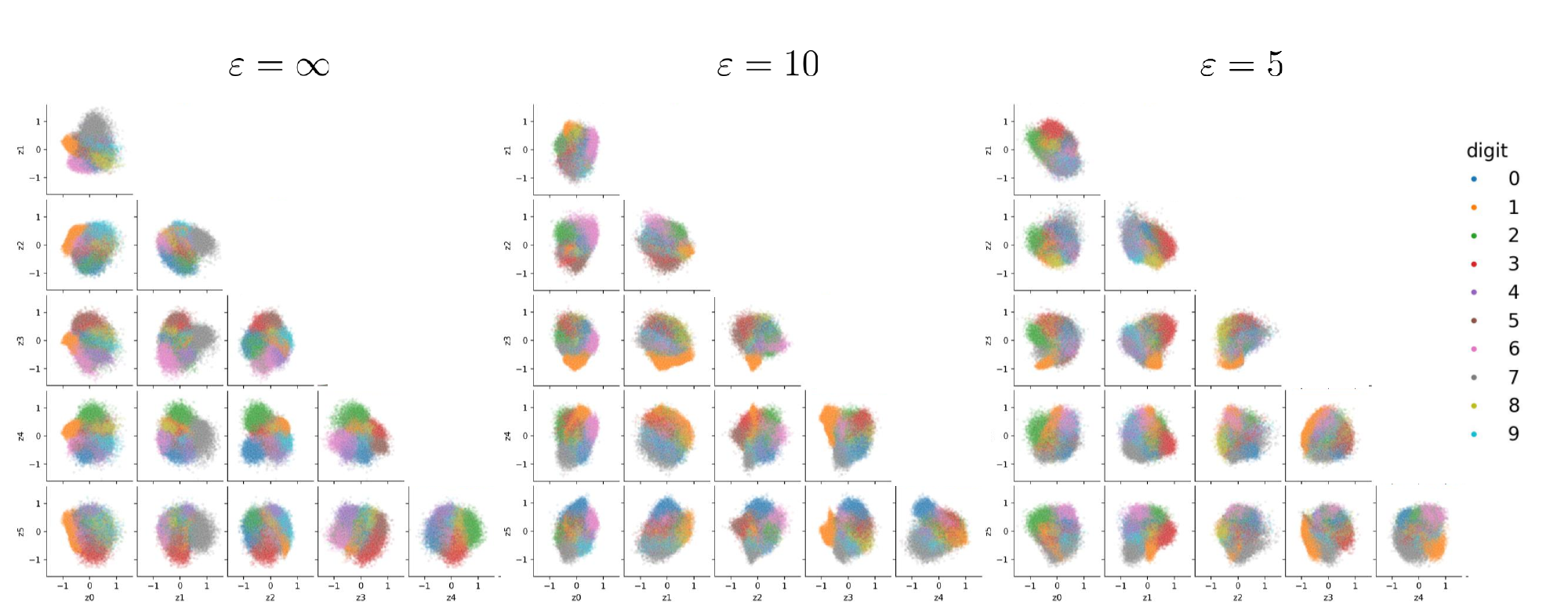}
    \end{minipage}
\end{figure}

\section{Fairness via private in-processing}\label{section:fairness}

Assume we are given a dataset $\vect D$ with $n$ samples of the form $(x_i,a_i, y_i)$ or $(x_i,a_i)$, where $x_i$ represents the non-sensitive attributes, $a_i \in \{0,1\}$ is a binary sensitive attribute, and $y_i$ is the response variable, available only in the supervised setting. Standard machine learning algorithms aim to minimize the empirical risk with respect to a given loss function $\ell$:
$ \min_{\theta} \mathscr L(\theta) \eqdef \min_{\theta} \frac{1}{n} \sum_{i=1}^n \ell(g_\theta(x_i)) \;,
$
where $\ell(g_\theta(x_i))$ is a shorthand for $\ell(g_\theta(x_i), y_i)$ in supervised tasks, and $\ell(g_\theta(x_i), x_i)$ in unsupervised tasks. %
Within our framework, different fairness notions can be promoted during training by incorporating appropriate penalty terms into the loss function $\mathscr L(\theta)$, while preserving bounded sensitivity for strong privacy guarantees.

Among the most common fairness notions, \emph{statistical parity} requires that the algorithmic decision is independent of the sensitive attribute. That is, statistical parity holds if $\mathcal L(g_\theta(X) \mid A=0) = \mathcal L(g_\theta(X) \mid A=1)$. Let $\vect X_j = \{x_i : a_i = j\}$ and $n_j = \textnormal{length}(\vect X_j)$ for $j = 0,1$.  Statistical parity can be encouraged by minimizing
\begin{equation}\label{eq:loss_SP_text}
    \mathscr L^{SP}_\alpha(\theta) = (1-\alpha)\mathscr L(\theta)  + \alpha \, {SW}_2^2\left(  \varphi_\theta \# P_{\vect X_0}, \varphi_\theta \# P_{\vect X_1} \right),
\end{equation}
where $\alpha \in [0,1]$ controls the trade-off between prediction accuracy and fairness, and  $\varphi_\theta= g_\theta$ (if the distribution loss is applied to the output)  or $g_\theta= \psi_\theta \circ \varphi_\theta$ (if the distribution loss is applied to intermediate representation, e.g., an intermediate layer in a NN). To ensure privacy guarantees, we must assume that $n_0$ and $n_1$ are fixed. Under suitable boundedness conditions, the sensitivity of the gradient $\nabla_\theta \mathscr L^{SP}_\alpha(\theta)$ (or its Monte Carlo approximation) is bounded by
$(1-\alpha) \frac{2C}{n} + \alpha \frac{16 M L }{\min \{n_0,n_1\}} ,
$
under the two-end neighboring relation $\sim_2$ (see Corollary \ref{th:fairandprivate}). This framework for private bias mitigation can be adapted to other fairness notions as detailed in Appendix~\ref{appendix:fairness}. Thanks to its generality, our methodology enables the development of private bias mitigation procedures for a wide range of learning tasks. 

To demonstrate the effectiveness and flexibility of our approach, we simulate biased data as described in Appendix~\ref{appendix:fairness}, and evaluate the performance of our private in-processing mitigation strategy across three distinct scenarios. First, we address the well-studied problem of private and fair binary classification, previously explored in several works (see extended related work, Appendix \ref{sec:ExtendedRelatedWork}),
and then, we introduce two novel tasks: multidimensional fair and private regression, extending beyond the current one-dimensional solutions (e.g., \cite{xian2024DPfairRegression}), and  fair and private representation learning. As a simple illustration of our methodology, Figure \ref{fig:fairness_try} represents the distribution of $g_\theta(X)|A=0$ versus the distribution of $g_\theta(X)|A=1$ for each of the three learning scenarios proposed, for different values of the regularization parameter $\alpha$ and the privacy budget $\varepsilon$, and fixed $\delta$. Two main conclusions can be drawn. First, the Wasserstein penalization
mitigates biases as $\alpha$ increases. Second, adding privacy does not significantly alter the results. Full experimental details, along with comprehensive statistical metrics for comparing the results, are provided in Appendix~\ref{appendix:fairness}.

\begin{figure}[ht]
    \centering
\includegraphics[width=1\linewidth]{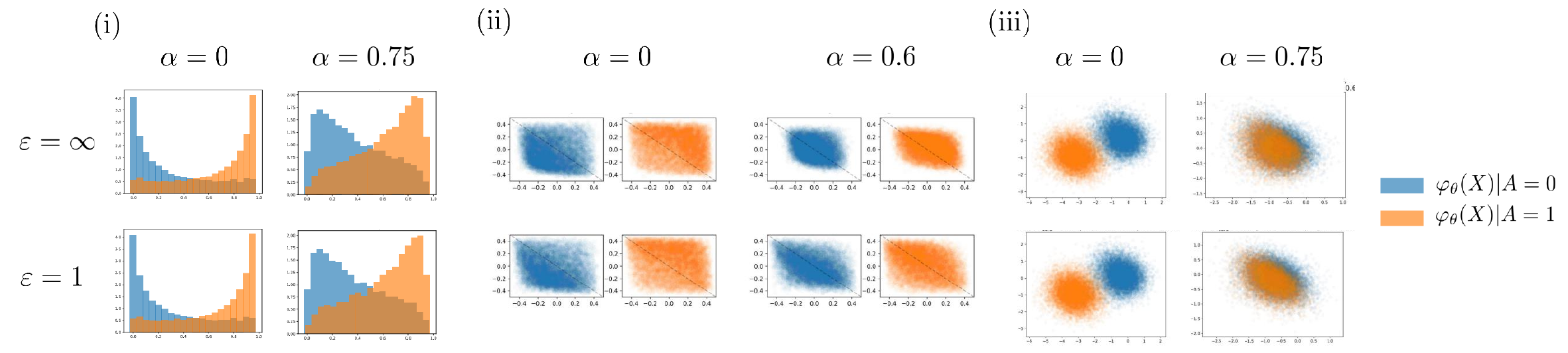}
    \caption{Distributions of $\varphi_\theta(X)$ conditioned on the sensitive attribute $A = 0$ and $A = 1$, for models trained to minimize the objective in~\eqref{eq:loss_SP_text}, across different values of the regularization parameter $\alpha$ and the privacy budget $\varepsilon$. The representation $\varphi_\theta$ corresponds to: (i) predicted class probabilities in a classification task, (ii) predicted values in a bidimensional regression task, and (iii) bidimensional latent representations from an autoencoder.}

    \label{fig:fairness_try}
\end{figure}

\section{Conclusion}
In this work, we introduced a novel sensitivity analysis for sliced Wasserstein gradients, enabling new applications in distributional matching problems under differential privacy. We proposed two innovative applications of this framework: a private learning algorithm for sliced Wasserstein autoencoders and a general in-processing method for private and fair learning. Across tested problems where comparable Wasserstein-based algorithms exist (see \Cref{section:swae} and \Cref{section:distribution_matching}), our approach consistently matches or surpasses existing methods. These results highlight the potential of our framework to advance the state of private and fair machine learning. They  open promising directions  for future research  in private distributional learning.

A limitation of our method is its potential computational overhead (see \Cref{remark:computations}), for which we do not foresee solutions beyond those already discussed in this article. However, in context, computational overheads are already common with differential privacy \cite{DBLP:journals/corr/abs-2204-13650}, and our method opens applications that were previously impossible.

This work opens many other research directions.  Our results do not directly extend to other Wasserstein losses, such as $W_p$. As shown in Appendix~\ref{counterexample},  it is not possible to bound the sensitivity of the gradient of $W_p$, for general $p \geq 1$, by a factor that decreases approximately at a rate of $1/n$. Nevertheless, we believe that the generalization of our method to $W_p^p$, $p>1$, is both feasible and worthwhile. A natural extension to the privacy of the computation of sliced Wasserstein barycenters is also left for future research. Finally, incorporating \emph{entropic regularization} may be interesting to reduce the computational complexity and favor statistical stability.

\section*{Acknowledgements}

This paper has been partially funded by the Agence Nationale de la Recherche under grant ANR-23-CE23-0029 Regul-IA. The research leading to these results received funding from MCIN/AEI/10.13039/501100011033/FEDER under Grant Agreement Number PID2021-128314NB-I00. The authors also acknowledge the support of the AI Cluster ANITI (ANR-19-PI3A-0004).

\bibliography{biblio}
\bibliographystyle{icml2025}

\appendix

\section{Extended Related Work}
\label{sec:ExtendedRelatedWork}

\paragraph{Fairness in Machine Learning.}
Fairness in machine learning has emerged as a critical area of research, driven  by the growing recognition of its societal impact and the ethical implications of algorithmic decision-making. Additionally, regulatory frameworks such as the General Data Protection Regulation (GDPR) and the recent European AI Act\footnote{\url{https://artificialintelligenceact.eu/}} mandate stringent measures to identify and mitigate bias in AI systems, emphasizing the need for fair and private methodologies in machine learning. Unfairness arises when certain variables, often referred to as sensible variable, systematically bias the behavior of an algorithm against specific groups of individuals, leading to disparate outcomes. This field of research has received a growing attention over the last few years as pointed out in the following papers and references therein \cite{chouldechova2020snapshot,dwork2012fairness,oneto2020fairness,wang2022brief,barocas2018fairness,besse2022survey}.

The Wasserstein distance offers a compelling framework for addressing fairness, as it provides a principled way to quantify discrepancies between the distributions of different subgroups. Moreover, as stated first in \cite{feldman2015certifying}, then in \cite{gouic2020projection} or \cite{chzhen2020fair}, Wasserstein distance between the conditional distributions of the algorithm  for each group, is the natural  measure to quantify the cost of ensuring fairness of the algorithm, defined as algorithms exhibiting the same behavior for each group. Hence optimal transport based methods are commonly used to  assess and mitigate distributional biases, paving the way for more equitable algorithmic decision-making. We refer, for instance, to the previously mentioned references \cite{silvia2020general,gordaliza2019obtainingFairness} and references therein.

\paragraph{Differential Privacy and Fair Learning.}

The interplay between fairness and differential privacy has received significant attention in recent years. A comprehensive review of this topic in decision and learning problems is provided in \cite{fioretto2022DPandfairness}. Within the learning framework, research has progressed in various directions. From a theoretical standpoint, despite the early work of \cite{cummings2019compatibility} demonstrating inherent incompatibilities between exact fairness and differential privacy, \cite{mangold2023boundedimpact} recently presented promising theoretical results indicating that fairness is not severely compromised by privacy in classification tasks. Another research direction has focused on studying the disparate impacts on model accuracy introduced by private training of algorithms. This phenomenon was first observed in \cite{bagdasaryan2019disparateimpact} and has been extensively studied in subsequent works \cite{farrand2020neithePrivateNorFair,tran21fairnesslens,xu2021removingDisparateImpact,esipova2023gradientMisalignment}. A third line of research aims to develop models that are both private and fair. Private and fair classification models have been proposed using in-processing and post-processing techniques across various scenarios in \cite{xu2019achievingDPandFairness,jagielski2019differentially,ding2020differentially,lowi2022SDPfairLearning,yaghini2023learning,Ghoukasian2024DPFairBinaryClassif}. A recent comparison of these works can be found in \cite{Ghoukasian2024DPFairBinaryClassif}. 
In the topic of fair and private regression, the only available work is the aforementioned post-processing method of \cite{xian2024DPfairRegression}.

\section{A Framework for Deep Learning}
\label{sec:DeepLearningFramework}

This section explains how to adapt the methods presented above into a deep-learning framework where it is typically not possible to guarantee a priori that the gradients and the activations are bounded, and where one typically needs to run many iterations in a batched setting.

\subsection{Inner-Clipping of the Gradients}

Directly applying \Cref{theorem:gradient_sensitivity} to general deep learning models is typically infeasible, as the required boundedness conditions are not satisfied. As a solution, we propose to introduce three hyperparameters $M \geq 0$, $L_1 \geq 0$, and $L_2 \geq 0$, and to use as a proxy for ${\nabla}_{\theta}W_2^2$ the following quantity:
\begin{equation}\label{eq:gradient_approx}
\begin{aligned}
       \nabla^{M,L_1,L_2}_{\theta}W_2^2 & = 2 \sum_{i=1}^n \sum_{j=1}^m R_{\sigma(i),\tau(j)} (U_i - V_j) \proj_{L_1} \p{ \nabla_{\theta}  g_{\theta}(x_i)} \\
        &+ 2 \sum_{i=1}^n \sum_{j=1}^m R_{\sigma(i),\tau(j)} (V_j - U_i) \proj_{L_2} \p{\nabla_{\theta}  h_{\theta}(z_j)}
\end{aligned}
\end{equation}
where for all i and j, we have $U_i = \proj_M \p{g_{\theta}(x_i)}$, $V_j = \proj_M \p{h_{\theta}(z_j)}$, and $\sigma, \tau$ are defined as in Section \ref{sec:WassersteinGradients}. This technique is known as \emph{``clipping"} and was historically introduced as a preprocessing of the gradients for problems with a finite-sum structure \cite{abadi2016deep}. For Wasserstein gradients, note that we also need to clip the \emph{activations}. Now, \Cref{theorem:gradient_sensitivity}  applies and one may add noise to this quantity to make it private with the Gaussian mechanism.

\subsection{Amplification by Subsampling}
In deep-learning, sub-sampling is often a necessity because of the size of the datasets. With differential privacy, it allows to leverage a property called \emph{privacy amplification by subsampling}. Since such property varies depending on the neighboring relation, we formalize it in the following lemma with the conventions of this article.

\begin{lemma}[Privacy amplification by subsampling] \label{lemma:subsampling}
    Let $n_1' \leq n_1, \dots, n_k' \leq n_k$. If a mechanism $M_{\text{batch}}$ is $(\varepsilon, \delta)$-DP on $\mathcal D_1^{n_1'} \times \ldots \times \mathcal D_k^{n_k'}$, the mechanism $M$ that 
    (i) selects $n_{i}'$ among the $n_i$ points in each category without replacement, and then
    (ii) applies $M_{\text{batch}}$ to the sampled dataset, is $(\varepsilon', \delta')$-DP on $\mathcal D_1^{n_1} \times \ldots \times \mathcal D_k^{n_k}$ where $\varepsilon' = \ln \p{1 + p\p{e^{\varepsilon} - 1}} , \quad \delta' = p \delta$ 
    and $p = \max \p{\frac{n_{1}'}{n_1}, \dots,  \frac{n_k'}{n_k}}$.
\end{lemma}

\subsection{Privacy Accountanting}
In the influential article \cite{abadi2016deep}, the authors introduce the \emph{moment accountant} method, a framework for quantifying the privacy of a composition of subsampled Gaussian mechanisms. We now detail why similar methods \cite{dong2019gaussian} are applicable to Wasserstein gradients.

\begin{algorithm}[ht]
\caption{Sequential Computation of Subsampled Wasserstein Noisy Gradients}
\label{alg:sequentialgradients}
\begin{algorithmic}
\For{$t = 1$ to $T$}
    \State Select $n_i'$ among the $n_i$ points in each category without replacement.
    \State Compute $\hat{\nabla}_{\theta}W_2^2 \eqdef {\nabla}_{\theta}^{M,L_1,L_2}W_2^2 + \sigma \mathcal{N}(0, I_d)$ on the subsampled dataset.
    \State Publish $\hat{\nabla}_{\theta}W_2^2$.
    \State Wait for the optimizer to update $\theta$.
\EndFor
\end{algorithmic}
\end{algorithm}

Since our neighboring relation is based on the replacement relation and since we use subsampling with fixed batch size and without replacement, the classical moment accountant method \cite{abadi2016deep} does not apply. We thus turn to the accounting techniques of \cite{dong2019gaussian} that build on the theory of $f$-differential privacy and that are more suited to this scenario.
Using the notations of \cite{dong2019gaussian}, and denoting by $\Delta$ the sensitivity of $\nabla^{M,L_1,L_2}_{\theta}W_2^2$ (which is controlled by \Cref{theorem:gradient_sensitivity}), $\nabla^{M,L_1,L_2}_{\theta}W_2^2$ is $\frac{\Delta}{\sigma}$-GDP (for \emph{Gaussian Differential Privacy}) ignoring the subsampling step.
In order to account for subsampling, one would like to apply Theorem 4.2 in \cite{dong2019gaussian}. This is not possible since this article uses a different neighboring relation and a different form of subsampling. However, we can notice that we can substitute Lemma 4.4 in the proof of Theorem 4.2 in \cite{dong2019gaussian} by our \Cref{lemma:subsampling} and the rest of the proof follows. We thus get that the overall procedure described by Algorithm 1 is $C_p(G_{(\sigma/\Delta)^{-1}})^{\otimes T}$-DP where $p = \max (n_{1}'/{n_1}, \dots, {n_k'}/{n_k})$ with the formalism of $f$-differential privacy \cite{dong2019gaussian}.

Finally, our approach aligns with the privacy accounting framework in the asymptotic regime described in Section 5.2 of \cite{dong2019gaussian}. In our experiments, we implement this privacy accountant using the \cite{opacus} library.

Note that this accountant can be replaced by any accountant, tailored for  fixed size  batch sampling without replacement and with the replacement neighboring relation.

\section{Extension to the sliced Wasserstein distance ($d \geq 2$)} \label{append:slice}
As pointed out in Remark \ref{remark:extension_sliced}, the results of this paper can be extended to higher dimensions by considering the sliced Wasserstein distance. Assume that $g_\theta(x),h_\theta(x)\in\mathbb R^d$. Following the notation of \cref{sec:SensitivityPrivacy}, we are interested now in bounding the sensitivity of the gradient of the (squared) sliced Wasserstein distance between the distributions $g_\theta\# P_{\vect{X}}$ and $h_\theta\# P_{\vect{Z}}$ in $\mathbb R^d$,  defined as 
\begin{align*}
    {SW}_{2}^2(&g_\theta\# P_{\vect{X}},h_\theta \# P_{\vect{Z}}) = \int_{\mathbb S^{d-1} }W_2^2\Bigl(\operatorname{Pr}_{\vartheta}\# (g_\theta\# P_{\vect{X}}),\operatorname{Pr}_{\vartheta}\# (h_\theta \# P_{\vect{Z}})\Bigr) d\mu(\vartheta) \ ,
\end{align*}
where $\mu$ represents the uniform measure on $\mathbb S^{d-1}$, the unit sphere of $\mathbb R^d$. From a practical standpoint, we are mainly interested in the study of the gradient of its Monte Carlo approximation given by $k$ i.i.d. random directions $\vartheta_1,\ldots,\vartheta_k\in\mathbb S^{d-1}$,
\begin{align*}
{SW}_{2,k}^2(&g_\theta\# P_{\vect{X}},h_\theta \# P_{\vect{Z}}) = \frac{1}{k} \sum_{l=1}^k W_2^2\Bigl(\operatorname{Pr}_{\vartheta_l}\# (g_\theta\# P_{\vect{X}}),\operatorname{Pr}_{\vartheta_l}\# (h_\theta \# P_{\vect{Z}})\Bigr) \ .
\end{align*}
As in the proof of \cref{theorem:gradient_sensitivity}, it suffices to bound the sensitivity of the gradient with respect to the substitution neighboring relation $\vect X \sim_1 \vect{\Tilde X}$. If we define $\Phi(\vect X)= \nabla_\theta {SW}_{2}^2(g_\theta\# P_{\vect{X}},h_\theta \# P_{\vect{Z}})$ and $\Phi_\vartheta(\vect X)= \nabla_\theta W_2^2(\operatorname{Pr}_{\vartheta}\# (g_\theta\# P_{\vect{X}}),\operatorname{Pr}_{\vartheta}\# (h_\theta \# P_{\vect{Z}}))$, by the chain rule and the same reasoning as in the proof of Theorem 1 in \cite{bonnel2015slicedRadon}, we know that under suitable smoothness assumptions, in the set of non-repeated points $\Gamma =\{\theta: g_\theta(x_i)\neq g_\theta(x_j) , h_\theta(z_i)\neq h_\theta(z_j) \textnormal{ for }i\neq j\} $,
\begin{equation*}
\begin{aligned}
    \Phi(\vect X) &= \int_{\mathbb S^{d-1}} \Phi_{\vartheta}(\vect X) d\mu(\vartheta) \ .
\end{aligned}
\end{equation*}
As in \cref{sec:WassersteinGradients}, we can define the \textit{gradient} $\Phi(\vect X)$ by this expression, even outside the set of differentiability points $\Gamma$, and provide privacy guarantees for every point.  Similarly, if we consider the Monte Carlo approximation of the gradient $\Phi(\vect X)= \nabla_\theta {SW}_{2,k}^2(g_\theta\# P_{\vect{X}},h_\theta \# P_{\vect{Z}})$, it follows that
$\Phi(\vect X) = \frac{1}{k}\sum_{l=1}^k \Phi_{\vartheta}(\vect X)$. In any case, we can conclude that 
\begin{equation*}
\begin{aligned}
    \Delta_2\Phi = \sup_{\vect X\sim \vect{\Tilde X}} \|\Phi(\vect X)-\Phi(\vect{\Tilde X})\|_2\leq \sup_{\vartheta\in\mathbb S^{d-1}} \Delta_2\Phi_\vartheta 
\end{aligned}
\end{equation*}

The sensitivity of $\Phi_\vartheta$ can be controlled with the one-dimensional results in Section \ref{sec:SensitivityPrivacy}. Note that if we define $ g_\theta^\vartheta(x)= (\operatorname{Pr}_{\vartheta}\circ g_\theta )(x)=  \vartheta^T g_\theta(x)$ and  $ h_\theta^\vartheta(z)= (\operatorname{Pr}_{\vartheta}\circ h_\theta )(z) =\vartheta^T h_\theta(z)$, then $\Phi_\vartheta(\vect X)= \nabla_\theta W_2^2(g_\theta^\vartheta\# P_{\vect{X}}, h_\theta^\vartheta \# P_{\vect{Z}})$, and we can conclude 

$$ \Delta_2\Phi_\vartheta   \leq  4M\frac{3L_1+L_2}{n}$$
provided that:
\begin{enumerate}
        \item[(i)] $|g_\theta^\vartheta(x)| =|\vartheta^T g_\theta(x)|\leq M$, $|h_\theta^\vartheta(x)|=|\vartheta^T h_\theta(z)|\leq M$ .
        \item[(ii)] $\|\nabla_\theta g_\theta^\vartheta(x)\|_2 =  \|\vartheta^T \mathcal J_\theta g_\theta(x)\|_2  \leq L_1$,  $\|\nabla_\theta h_\theta^\vartheta(z)\|_2=\|\vartheta^T \mathcal J_\theta h_\theta(z)\|_2 \leq L_2$.
\end{enumerate} 
In particular, both inequalities are verified uniformly in $\vartheta$ if we impose the following, more natural conditions: 
\begin{enumerate}
\item[(i)] $\|g_\theta^\vartheta(x)\|_2\leq M$, $\|h_\theta(z)\|_2\leq M$
 \item[(ii)] $\| \mathcal J_\theta g_\theta(x)\|_2= \underset{\|\eta\|_2=1}{\sup} \| \mathcal J_\theta g_\theta(x) \eta \|_2  \leq L_1$, $\| \mathcal J_\theta h_\theta(z)\|_2= \underset{\|\eta\|_2=1}{\sup} \| J_\theta h_\theta(z) \eta \|_2  \leq L_2$.
\end{enumerate} 
 By the properties of the spectral norm, $\|\mathcal J_\theta g_\theta(x)\|_2 = \|\mathcal J_\theta g_\theta(x)^T \|_2$. Therefore, for every $\vartheta\in\mathbb S^{d-1}$ and $x$, 
 \begin{equation*}
     \|\nabla_\theta g_\theta^\vartheta(x)\|_2 =  \|\vartheta^T \mathcal J_\theta g_\theta(x)\|_2 = \| \mathcal J_\theta g_\theta(x)^T \vartheta\|_2  \leq L_1 \ ,
 \end{equation*}
and similarly for $h_\theta$. As in the one dimensional setting, assumption (ii) is verified if $g_\theta$ and $h_\theta$ are $L_1$-Lipschitz and  $L_2$-Lipschitz with respect to $\theta$. To see this, note that if $\|\eta\|_2=1$, by the Lipschitz condition,
\begin{align*}
     \|\mathcal J_\theta g_\theta(x) \eta \|_2 &= \Big\| \lim_{t\rightarrow0} \frac{g_{\theta+t\eta}(x)-g_\theta(x)}{t} \Big\| \leq L_1\|\eta\|_2 = L_1
\end{align*}
Therefore, Theorem \ref{theorem:gradient_sensitivity} can be extended to the multidimensional setting with the sliced Wasserstein distance as follows:
\begin{theorem}\label{theorem:gradient_sensitivity_sliced}
    With all the previous notation, assume that there exists constants $M,L_1,L_2\geq 0$ such that for each $\theta\in\Theta$, $x\in\mathcal X$ and $z\in\mathcal Z$,
    \begin{enumerate}
\item[(i)] $\|g_\theta^\vartheta(x)\|_2\leq M$, $\|h_\theta(z)\|_2\leq M$
 \item[(ii)] $\| \mathcal J_\theta g_\theta(x)\|_2= \underset{\|\eta\|_2=1}{\sup} \| \mathcal J_\theta g_\theta(x) \eta \|_2  \leq L_1$, $\| \mathcal J_\theta h_\theta(z)\|_2= \underset{\|\eta\|_2=1}{\sup} \| J_\theta h_\theta(z) \eta \|_2  \leq L_2$.
\end{enumerate} 
    Then, 
 \begin{itemize}
\item[(a)]  Under neighboring relation $\sim_1$ in $\mathcal D = \mathcal X^n$, if we define $\Phi_\theta(\vect X)$ as $\nabla_{\theta} SW_{2}^2(g_\theta\# P_{\vect X}, h_\theta\# P_{\vect Z})$ or its Monte Carlo approximation $\nabla_{\theta} SW_{2,k}^2(g_\theta\# P_{\vect X}, h_\theta\# P_{\vect Z})$  then
$$ \Delta_2\Phi_\theta   \leq  4M\frac{3L_1+L_2}{n} \ .$$
\item[(b)] Under neighboring relation $\sim_2$ in $\mathcal D = \mathcal X^n\times \mathcal Z^m$, if we define $\Psi_\theta(\vect X, \vect Z)$ as $\nabla_{\theta} SW_{2}^2(g_\theta\# P_{\vect X}, h_\theta\# P_{\vect Z})$ or its Monte Carlo approximation $\nabla_{\theta} SW_{2,k}^2(g_\theta\# P_{\vect X}, h_\theta\# P_{\vect Z})$, then
$$ \Delta_2\Psi_\theta   \leq  4M\max\Bigl\{ \frac{3L_1+L_2}{n}  , \frac{L_1+3L_2}{m}  \Bigr\}\ .$$
    \end{itemize}
\end{theorem}
\vspace{0.5cm}

\begin{remark}\label{remark:degraded_bound}
From a computational point of view, if we want to define a clipped approximation  $\mathcal J^{L_1}_\theta g_\theta(x)$  of $\mathcal J_\theta g_\theta(x)$ that verifies Assumption (ii) in Theorem \ref{theorem:gradient_sensitivity_sliced}, this might be done by clipping the eigenvalues of the singular value decomposition of $\mathcal J_\theta g_\theta(x)$. This should be done at each step, for each $x_i$ in the batch. To simplify the computation and enable easy parallelization, we have adopted a suboptimal, naive alternative approach. Given $g_\theta = (g_\theta^1,\ldots, g_\theta^d)$, we define 
\begin{equation*}
    \mathcal J^{L_1}_\theta g_\theta(x) = \left( \begin{array}{c}
         \operatorname{clip}_{\frac{L_1}{\sqrt{d}}} (\nabla_\theta g_\theta^1(x) ) \\
          \vdots \\
            \operatorname{clip}_{\frac{L_1}{\sqrt{d}}} (\nabla_\theta g_\theta^d(x))
    \end{array}\right) \ .
\end{equation*}
Then, it follows that for every $\eta$ verifying  $\|\eta\|_2 = 1$,  
$$\| \mathcal J^{L_1}_\theta g_\theta(x)\eta \|_2= \Bigl( \sum_{i=1}^d \langle \operatorname{clip}_{\frac{L_1}{\sqrt{d}}} (\nabla_\theta g_\theta^i(x)), \eta \rangle^2 \Bigl) ^{1/2}  \leq \Bigl( \sum_{i=1}^d \frac{L_1^2}{{d}}\Bigr)^{1/2} =  L_1$$
\end{remark}
which implies the desired bound on the spectral norm, $\| \mathcal J^{L_1}_\theta g_\theta(x)\|_2 \leq L_1.$

\section{Additional details on the private sliced Wasserstein autoencoder}\label{appendix:SWAE}

\begin{table}[H]
\centering
\begin{tabular}{p{6.5cm}p{6.5cm}
}
\hline
\textbf{Encoder} & \textbf{Decoder} \\
\hline
Input: $1 \times 28 \times 28$ image & Input: latent vector ($\mathbb{R}^{\text{latent\_dim}}$) \\
Conv2D (8 filters, $3 \times 3$), LeakyReLU (0.2) & FC: latent\_dim $\rightarrow$ 64, ReLU \\
AvgPool ($2 \times 2$) $\rightarrow 8 \times 14 \times 14$ & FC: 64 $\rightarrow$ 128, ReLU \\
Conv2D (16 filters, $3 \times 3$), LeakyReLU (0.2) & FC: 128 $\rightarrow$ 784, ReLU \\
AvgPool ($2 \times 2$) $\rightarrow 16 \times 7 \times 7$ & Reshape to $16 \times 7 \times 7$ \\
Conv2D (16 filters, $3 \times 3$), LeakyReLU (0.2) & Conv2D (16 filters, $3 \times 3$), LeakyReLU (0.2) \\
Flatten $\rightarrow \mathbb{R}^{784}$ & Upsample $\rightarrow 16 \times 14 \times 14$ \\
FC: 784 $\rightarrow$ 128, ReLU & Conv2D (8 filters, $3 \times 3$), LeakyReLU (0.2) \\
FC: 128 $\rightarrow$ 64, ReLU & Upsample $\rightarrow 8 \times 28 \times 28$ \\
FC: 64 $\rightarrow$ latent\_dim & Conv2D (1 filter, $3 \times 3$), Sigmoid \\
\hline
\end{tabular}
\vspace{0.2cm}
\caption{Architecture of the convolutional autoencoder.}
\label{tab:autoencoder_architecture}
\end{table}

\begin{figure}[H]
    \centering
\includegraphics[width=1\linewidth]{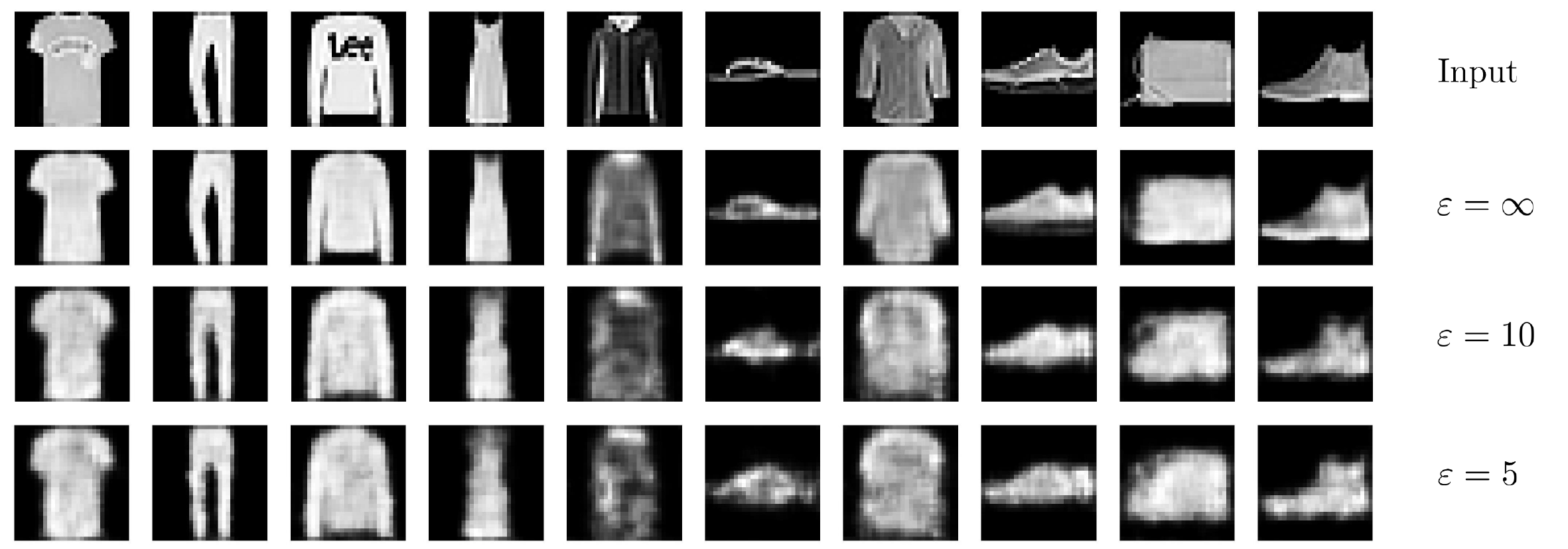}
    \caption{Reconstructed FASHION digits from the test dataset. The first row shows the earliest sample of each digit (0–9) in the test set. Subsequent rows display the corresponding reconstructions produced by the trained autoencoder under $(\varepsilon,\delta)$-DP, with $\delta = 10^{-5}$ and varying 
    values of $\varepsilon$.}
    \label{fig:fashion_reconstruction}
\end{figure}

\begin{figure}[H]
    \centering
\includegraphics[width=1\linewidth]{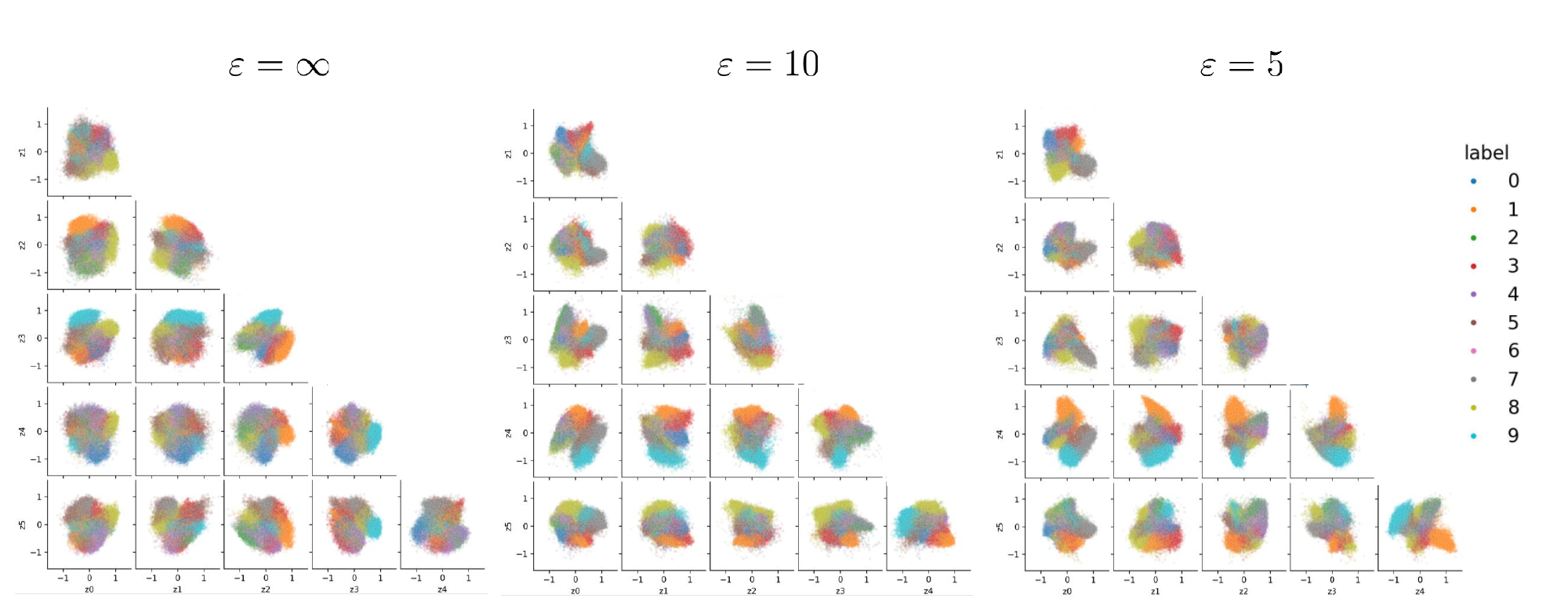}
    \caption{Encoded latent space FASHION samples for the autoencoder,  trained under $(\varepsilon,\delta)$-DP, with $\delta = 10^{-5}$ and varying 
    values of $\varepsilon$.}
    \label{fig:fashion_encoding}
\end{figure}

\begin{figure}[H]
    \centering
\includegraphics[width=1\linewidth]{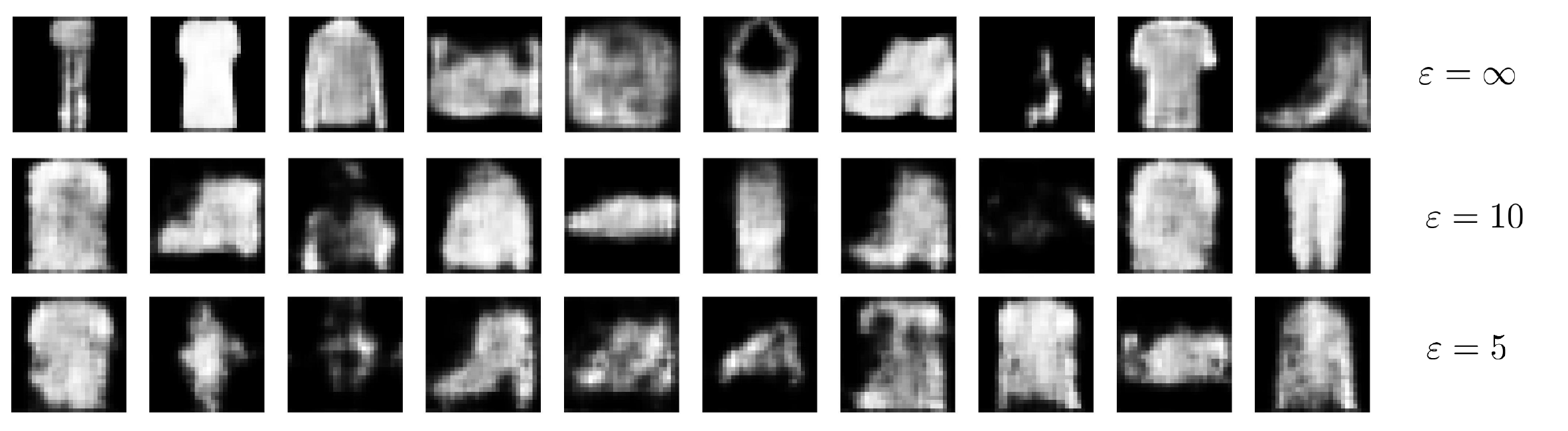}
    \caption{Generated samples from the FASHION autoencoder trained under $(\varepsilon,\delta)$-DP, with $\delta = 10^{-5}$ and varying 
    values of $\varepsilon$.}
    \label{fig:generated_fashion}
\end{figure}

\section{Fairness via private in-processing}\label{appendix:fairness}

\subsection{General method}

Let us begin by recalling the general framework introduced in Section~\ref{section:fairness}. Let $ \vect D$ denote a dataset with $n$ samples $(x_i,a_i, y_i)$ or $(x_i,a_i)$, where $x_i$ are the non-sensitive attributes, $a_i\in \{0,1\}$ is the sensitive attribute and $y_i$ the response variable, only available in supervised problem. 
We consider the general empirical risk minimization problem:
\begin{equation}\label{loss:ERM}
    \min_{\theta} \mathscr L(\theta) \eqdef \min_{\theta} \frac{1}{n} \sum_{i=1}^n \ell(g_\theta(x_i)) \ ,
\end{equation}
where $\ell(g_\theta(x_i))$ is a shorthand for $\ell(g_\theta(x_i), y_i)$ in supervised problems and  $\ell(g_\theta(x_i), x_i)$ in unsupervised problems. The goal is to privately optimize a regularized version of this problem, incorporating a fairness penalty based on the sliced Wasserstein distance. As explained in Section \ref{section:swae}, the  finite-sum structure of the loss function translates to the gradient, and as long as $\|\nabla_\theta \ell(g_\theta(x))\|_2 \leq C$ for all $x\in \mathcal X$, the sensitivity of $\nabla_\theta \mathscr L(\theta)$ is bounded by $2C/n$, not only for the substitution relation $\sim_1$, but  for any k-end neighboring relation $\sim_k$. To choose an adequate fairness penalty, we have considered in this work two different fairness notions: \textit{Statistical parity} and \textit{Equality of Odds}. The first was already discussed in Section \ref{section:fairness}, we include it again for completeness.
We present this section for the general case of the sliced Wasserstein distance, note that the case $d=1$ coincides with the one-dimensional Wasserstein distance. As in Section~\ref{section:fairness}, let $\varphi_\theta$ denote the representation to which the fairness penalty is applied. That is, $\varphi_\theta = g_\theta$ if the penalty is applied at the output, or $g_\theta = \psi_\theta \circ \varphi_\theta$ if it is applied to an intermediate representation (e.g., an intermediate layer of a neural network). In either case, for any distributions $P,Q$, the equality  $\mathcal L( \varphi_\theta \# P ) = \mathcal L( \varphi_\theta \# Q )$ implies $\mathcal L( g_\theta \# P ) = \mathcal L( g_\theta \# Q )$, which justifies penalizing any intermediate representation. 

\paragraph{Statistical Parity (SP):} Statistical parity corresponds to the situation where  the algorithmic decision does not depend on the sensitive variable.  Statistical parity is thus satisfied if  $ \mathcal L( g_\theta(X) | A= 0 ) = \mathcal L ( g_\theta(X) | A=1 )
    $. Given our data, if we define $\vect X_j = (x_i: a_i=j)$, $n_j = \textnormal{length}(\vect X_j)$ for $j=0,1$, statistical parity can be favored by minimizing   
    \begin{equation}\label{eq:loss_SP}
    \begin{gathered}
    \mathscr L^{SP}_\alpha(\theta) = (1-\alpha)\mathscr L(\theta)  + \alpha {SW}_2^2\Bigl(  \varphi_\theta \# P_{\vect X_0}, \varphi_\theta \# P_{\vect X_1} \Bigr)  
    \end{gathered}
    \end{equation}
    where $\alpha \in [0,1]$ measures the weight of each part in the optimization. In order to establish privacy guarantees, we need to assume that $n_0$ and $n_1$ are fixed.
    
\paragraph{Equality of Odds (EO):} Beyond guaranteeing the same decision for all, which is not suitable in some cases where the sensitive variable impacts the decision, bias mitigation may require that the model performs with the same accuracy for all groups, often referred to as equality of odds. We focus only in the supervised case, where $y_i$ is available and takes values in $\{0,\ldots,R-1\}$. In this case, equality of odds is verified if 
$ \mathcal L( g_\theta(X) | A= 0, Y = k ) = \mathcal L ( g_\theta(X) | A=1 , Y = k)$ for all $k \in \{0,\ldots,R-1\}$. With the same ideas as before, if we define $\vect X_{j,k} = (x_i: a_i=j,y_i=k)$, $n_{j,k} = \textnormal{length}(\vect X_{j,k})$ for $j \in \{0,1\}, k\in\{0,\ldots,R-1\}$,
equality of odds bias mitigation can be enforced by training with loss
\begin{equation}\label{eq:loss_EOO}
  \mathscr L^{EO}_\alpha(\theta)  = (1-\alpha)\mathscr L(\theta) + \frac{\alpha}{R} \sum_{k=1}^R {SW}_2^2\Bigl(  \varphi_\theta \# P_{\vect X_{0,k}}, \varphi_\theta \# P_{\vect X_{1,k}} \Bigr) 
\end{equation}
To obtain privacy guarantees, now we need to impose that the values $n_{j,k}$ are fixed. \\

\begin{corollary} \label{th:fairandprivate}
        If $\varphi_\theta$ verifies {$\|\varphi_\theta(x)\|_2 \leq M$, $\|\mathcal J_\theta \varphi_\theta(x)\|_2 \leq L$} and $\|\nabla_\theta \ell(g_\theta(x))\|_2 \leq C$, then 

    $\bullet$ For SP, under $\sim_2$, the sensitivity of  $\nabla_\theta \mathscr L^{SP}_\alpha(\theta)$ {or its MC approximation} is bounded by     
    \begin{equation}\label{eq:sensitivity_SP}
     (1-\alpha) \frac{2C}{n} + \alpha \frac{16 M L }{\min \{n_0,n_1\}} \ .
\end{equation}        
    $\bullet $ For EO, under $\sim_{2R}$, the sensitivity of $\nabla_\theta \mathscr L^{EO}_\alpha(\theta)$  {or its MC approximation} is bounded by     
    \begin{equation}\label{eq:sensitivity_EO}
     (1-\alpha) \frac{2C}{n} + \frac{\alpha}{R} \frac{16 M L }{\min_{j,k} \{n_{j,k}\}} \ .
    \end{equation} 
    \end{corollary}
    
\begin{remark}
Our privacy guarantees in the fairness framework are built upon the knowledge of class sizes. The importance of controlling these sizes has been previously recognized. For example, \cite{lowi2022SDPfairLearning} imposes a restriction on the minimum proportion of elements in each class, while \cite{Ghoukasian2024DPFairBinaryClassif} and \cite{xian2024DPfairRegression} derive privacy guarantees that degrade with smaller class sizes. Conceptually, our framework for establishing privacy guarantees is very sound. Even though an attacker might learn  some information about the number of individuals in each class used during training, they cannot distinguish between the outputs of two datasets $\vect D$ and $\vect{\Tilde D}$ differing only in one individual from the same class.
\end{remark}

\subsection{Applications: Private bias mitigation in diverse learning tasks}

In order to demonstrate the versatility of our methodology for imposing fairness in different scenarios, we use an illustrative synthetic model to simulate bias in algorithmic decision-making.
Note that we do not provide comparisons with other application-specific methodologies, as our approach is highly general and does not include any of the statistical, convergence, or fairness guarantees described by other methods, see \cite{xu2019achievingDPandFairness,jagielski2019differentially,ding2020differentially,lowi2022SDPfairLearning,yaghini2023learning,Ghoukasian2024DPFairBinaryClassif} for the fair and private classification problem, or \cite{xian2024DPfairRegression} for fair and private one-dimensional regression.  Yet we provide, to the best of our knowledge, the first method to handle  novel problems such as multidimensional fair and private regression, or fair and private representation learning.

We consider $\vect D =\{(x_i,a_i,y_i^C,y_i)\}_{i=1}^n$ i.i.d. samples with the same distribution as $(X,A,Y^C,Y)$, where $X$ denotes the features, $A$ is the sensitive variable, $Y^C=(Y^C_1,Y^C_2)$ is a continuous response variable and $Y$ is a discrete version of $Y^C$, related by

\begin{enumerate}
    \item[(i)] $Y^C \sim U([0,1]\times [0,1])$
    \item[(ii)] $Y = I\Bigl(Y^C_2 >1 - Y^C_1 \Bigr)$ 
    \item[(iii)] $A = BY + (1-B)(1-Y)$, where $B\sim \text{Bernoulli}(p)$ independent of $Y$.
    \item[(iv)] $X_{core} = [ \underbrace{Y^C,\ldots,Y^C}_{d_{core}/2 \  times}] + N\bigl(0,\sigma^2_{core} I_{d_{core}}\bigr)$,   $\quad X_{sp} = [ \underbrace{A,\ldots,A}_{d_{sp} \  times}] +  N\bigl(0,\sigma^2_{sp}\ I_{d_{sp}}\bigr)$.
    \item[(v)] $X = [X_{core},X_{sp}]$
\end{enumerate}

Therefore, this generated data consists in a response variable $Y_C$, which is correlated with the sensitive attribute $A$. The features $X$ are divided into two parts: a first part $X_{core}$ which is a noisy transformation of $Y_C$, and a second spurious part $X_{sp}$ which is a noisy version of $A$. If $p$ is close to 1, most of the cases verify $A=Y$ and therefore, the decision of the algorithm relies highly  on the sensitive variable $A$. Bias in the algorithmic decision is created when the sensitive variable $A$ is not aligned with the decision. When $Y \neq A$, the learning task is more complicated since while $X_{core}$ is correlated with $Y$, the spurious part pushes towards the bad decision. This setting reproduces the characteristics of some of the main biases present in many data sets, for instance,  \cite{adult_2} or \cite{statlog_(german_credit_data)_144}  in supervised learning. We explore this problem across different scenarios, demonstrating how penalized models using our Wasserstein--based losses can help mitigate unfairness, according to various fairness notions, while maintaining differential privacy guarantees. All experiments are conducted on the same synthetic dataset, generated with the previous mechanism and values
 $n=30000$, $p=0.7$, $d_{core}=d_{sp}=8$, $\sigma^2_{core} = 1/5$, $\sigma^2_{sp} = 2/5$. \\

All  models are trained with DP-SGD as explained in \cref{sec:DeepLearningFramework}, with clipping constant $C>0$ for the individual gradients in \eqref{loss:ERM}, as usual in DP-SGD, and inner clipping constants $M,L>0$ for the Wasserstein gradient approximation \eqref{eq:gradient_approx} or its sliced version. In the latter case, all the experiments use the naive clipping procedure explained in \cref{remark:degraded_bound}. \cref{th:fairandprivate} and the procedure described in Section~\ref{sec:DeepLearningFramework}, enable us to compute the privacy budget obtained after $T$ iterations of DP-SGD. In particular, in all the experiments, we fix the number of iterations $T$ and the value of $\delta$, and compute the required noise to obtain $(\varepsilon,\delta)$-DP after $T$ iterations, for different values of the privacy budget $\varepsilon$ and the weight $\alpha\in[0,1]$ in the penalized loss functions \eqref{eq:loss_SP} and \eqref{eq:loss_EOO}.

Following the above notation,  $\vect X_j = (x_i: a_i=j)$, $n_j = \textnormal{length}(\vect X_j)$ for $j=0,1$, and $\vect X_{j,k} = (x_i: a_i=j,y_i=k)$, $n_{j,k} = \textnormal{length}(\vect X_{j,k})$ for $j,k  \in \{0,1\}$. Given our data generation procedure, we know that $\mathbb E(n_j)=n/2$, $\mathbb E(n_{j,j}) = pn/2$ and $\mathbb E(n_{j,1-j}) = (1-p)n/2$ for $j\in\{0,1\}$. In all the subsequent experiments, the batch sizes considered are $n'_j \approx n_j/5$ when minimizing \eqref{eq:loss_SP}, and $n'_{j,k} \approx n_{j,k}/5$ when minimizing \eqref{eq:loss_EOO}, where the approximation is related to internal parallelization of the gradient computations in the code detailed in Section \ref{sec:computational_details}. A summary of the sample sizes in the data and the batch sizes considered is presented in Table \ref{tab:sample_sizes}.  \\

\begin{table}[h]
\centering
\begin{tabular}{c|cc|c}
\toprule
 & \(Y=0\) & \(Y=1\) & Total \\
\midrule
\quad A=0 & 10608 / 1050 & 4564 / 450 & 15172 / 1500 \\
\quad A=1 & 4446 / 400 & 10382 / 1000 & 14828 / 1450 \\
\bottomrule
\end{tabular}
\vspace{0.3cm}
\caption{Sample sizes / batch sizes for different groups and subgroups. The center part displays the pairs  $n_{j,k}/n'_{j,k}$, while the right side displays the pairs $n_{j}/n'_{j}$. }
\label{tab:sample_sizes}
\end{table}

\subsection{Classification} 

First, we consider the problem of predicting the label $Y$ as a function of $X$. Our decision rule is based on logistic regression, where the function $g_\theta$ maps each $x_i$ to the predicted probability $g_\theta(x_i)\in (0,1)$. The classification rule is given by $G_\theta(x) = I(g_\theta(x)>1/2)$. $g_\theta$ is defined as a neural network with one layer and a sigmoid activation function, and the classification loss is defined as the binary cross-entropy. In or experiments, we have imposed fairness using the two notions presented above. Considering $\varphi_\theta = g_\theta$, we have trained our model with simple iterations of DP-SGD to minimize \eqref{eq:loss_SP} to enforce SP, and \eqref{eq:loss_EOO} to enforce EOO. In both cases, we have trained the model for different values of the weight $\alpha$ and the privacy budget $\varepsilon$, when we fix $\delta = 0.1/n$,  number of iterations $T=500$, clipping constants $C=5$, $M=L=1$  learning rate $=0.05$ and number of projections in the Monte Carlo approximation $= 50$. For each fairness notion, we present a brief overview of the methodology, followed by a detailed exposition of the results. We include statistical measures to assess the trade-off between fairness, privacy, and utility for each model.

\paragraph{Statistical parity:} Statistical parity in classification, it is usually measured by the Disparate Impact, defined as \begin{equation}\label{eq:disparate:impact}
        DI(G_\theta) = \frac{\mathbb P (G_\theta(X) = 1 | A = 0)}{\mathbb P (G_\theta(X) = 1 | A = 1)} \ .
    \end{equation}
    However, enforcing statistical parity by enforcing independence between $G_\theta(X)$ and $A$ often produces unstable solutions as discussed in \cite{krco2023mitigating} or \cite{barrainkua2024uncertainty}, hence many authors propose to mitigate not only the mean but the whole distribution of the predicted probabilities $g_\theta(X) \in (0,1)$ as in \cite{risser2022tacklingAlgorithm}, \cite{gouic2020projection} or \cite{chzhen2020fair}, the same approach we have adopted in this work.
    
    Figure \ref{fig:di} shows the results of training the model minimizing \eqref{eq:loss_SP} for different values of the weight $\alpha$ and the privacy budget $\varepsilon$.
    Two main conclusions can be drawn. First, it confirms that the Wasserstein penalization approach mitigates the unfair biases present in the data set. We can see that, for increasing values of $\alpha$, the histograms of the scores conditioned on the value of the sensitive variable get closer, leading to a progressive reduction of biases, as seen in the decreasing values of the disparate impact, albeit at the expense of accuracy, as expected. The second important conclusion is that adding privacy does not significantly alter the results of the optimization. For different privacy budgets $\varepsilon$, both the histogram and the computed measures do not change much across the rows. Moreover, Figure \ref{fig:di_losses} shows the training loss curve of the optimization for each value of $\alpha$ considered when $\varepsilon$ varies. Low values of $\varepsilon$ result in noisier versions of the loss curve, but they remain very close to the non-private counterpart.

\begin{figure}[h!]
    \centering
\includegraphics[width=1\linewidth]{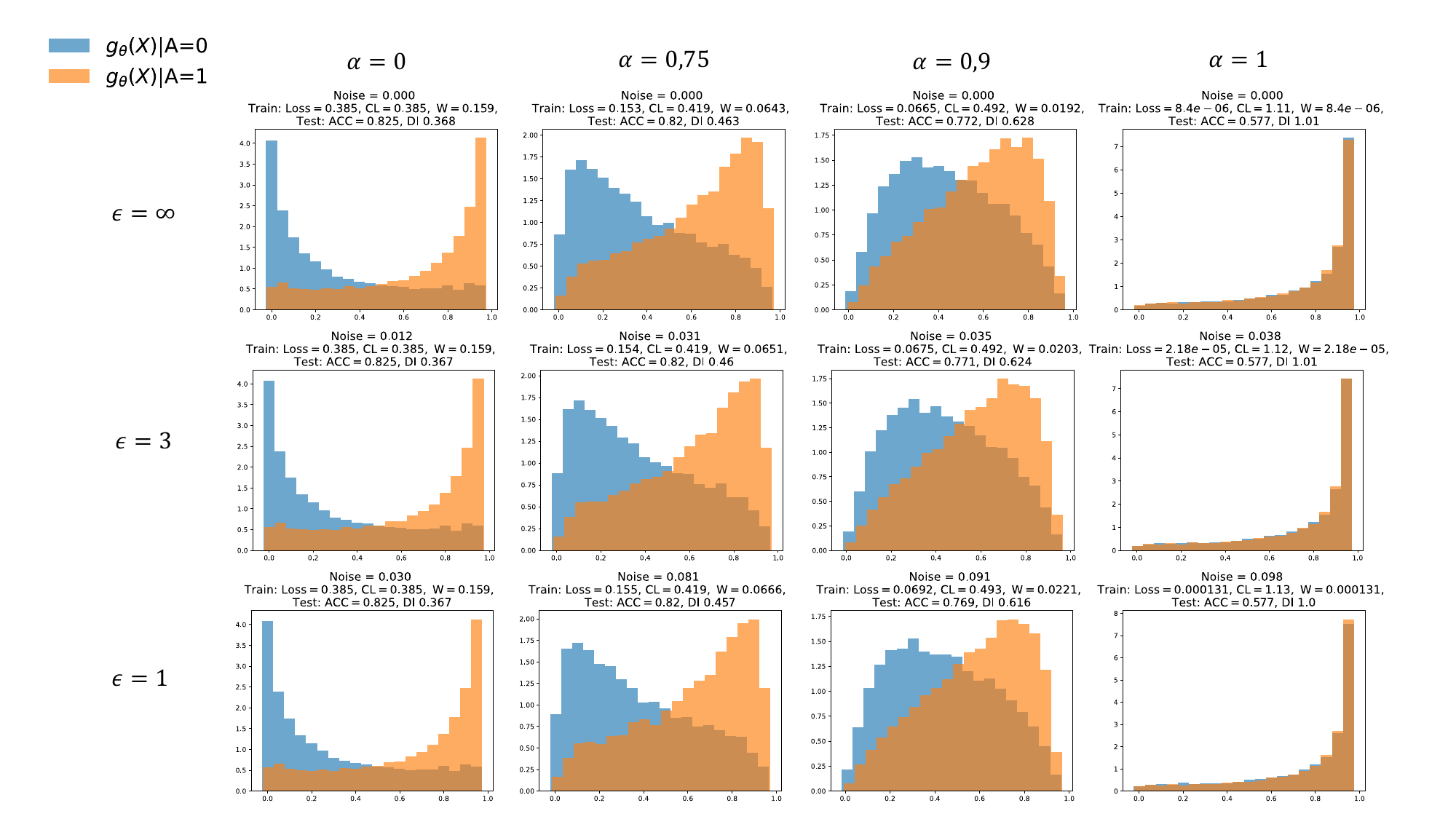}
    \caption{
    Histogram of the predicted probabilities $g_\theta(X)$ conditioned on the values of $A$ in the private classification model with SP regularization, for the different values of $\alpha$ and $\varepsilon$, and fixed $\delta = 0.1/n$. Above each graph we indicate the noise added at each step of DP-SGD to obtain the desired privacy level, the value of the loss \eqref{eq:loss_SP} in the training procedure, together with the individual value of the classification loss (CL) and the distributional Wasserstein loss (W). Last line includes accuracy (ACC) and disparate impact (DI) of the classification rule $G_\theta$ computed with independent test data.}
    \label{fig:di}
\end{figure}

\begin{figure}[h!]
    \centering
    \includegraphics[width=\linewidth]{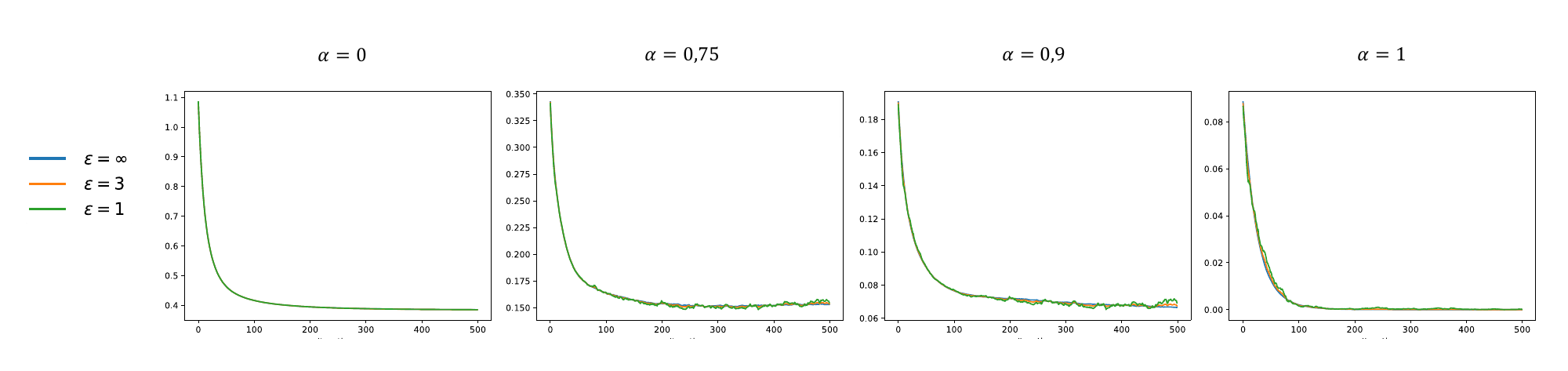}
    \caption{Training loss curve for the experiment of Figure \ref{fig:di}. Each graph represents the training loss \eqref{eq:loss_SP}  for a fixed value of $\alpha$ along the iterations of DP-SGD, for the different values of $\varepsilon$. }
    \label{fig:di_losses}
\end{figure}
  
\paragraph{Equality of odds:} %
To measure the unfairness of our a classification rule $G_\theta$, we consider the two following indexes:
    \begin{align}\label{eq:equality_odds_index}
         EO_1(G_\theta)& = \frac{\mathbb P (G_\theta(X) = 1 | A = 0, Y=1)}{\mathbb P (G_\theta(X) = 1 | A = 1, Y=1)} \ . \\
         EO_0(G_\theta)& = \frac{\mathbb P (G_\theta(X) = 1 | A = 0, Y=0)}{\mathbb P (G_\theta(X) = 1 | A = 1, Y=0)} \ . 
    \end{align}

Figure \ref{fig:eoo} shows the results of optimizing \eqref{eq:loss_EOO} for different values of $\alpha$ and  $\varepsilon$. 
The histogram of the distribution of the predicted probabilities, $g_\theta(X)|A=1,Y=j$ versus $g_\theta(X)|A=0,Y=j$, illustrates the model's capability to minimize discrepancies between the distributions as $\alpha$ increases, as shown by the values of $EO_0,EO_1$. From Figure \ref{fig:eoo}, we can also observe that private training has minimal impact on the model's fit across all values of $\alpha$. Similarly, it does not significantly affect the learning loss curve during optimization, as shown in Figure \ref{fig:eoo_losses}.

\begin{figure}[ht]
    \centering
\includegraphics[width=1\linewidth]{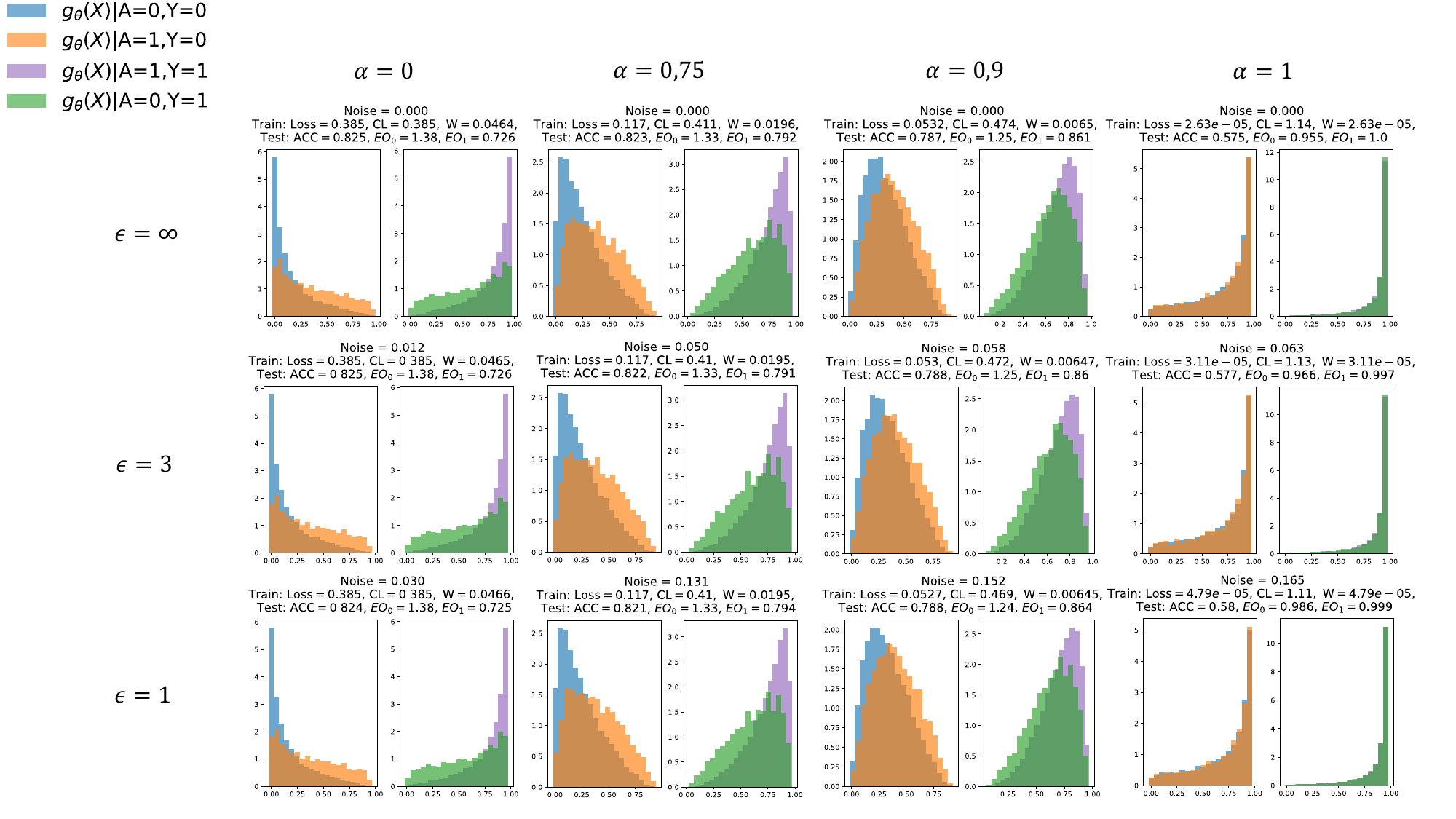}
    \caption{Histogram of the predicted probabilities $g_\theta(X)$ conditioned on the values of $A$ in the private classification model with EOO regularization, for the different values of $\alpha$ and $\varepsilon$, and fixed $\delta = 0.1/n$. Above each graph we indicate the noise added at each step of DP-SGD to obtain the desired privacy level, the value of the loss \eqref{eq:loss_EOO} in the training procedure, together with the individual value of the classification loss (L) and the  Wasserstein loss (W). Last line includes accuracy, $EO_0$ and $EO_1$ indexes computed with independent test data.}
    \label{fig:eoo}
\end{figure}

\begin{figure}[h!]
    \centering
    \includegraphics[width=0.88\linewidth]{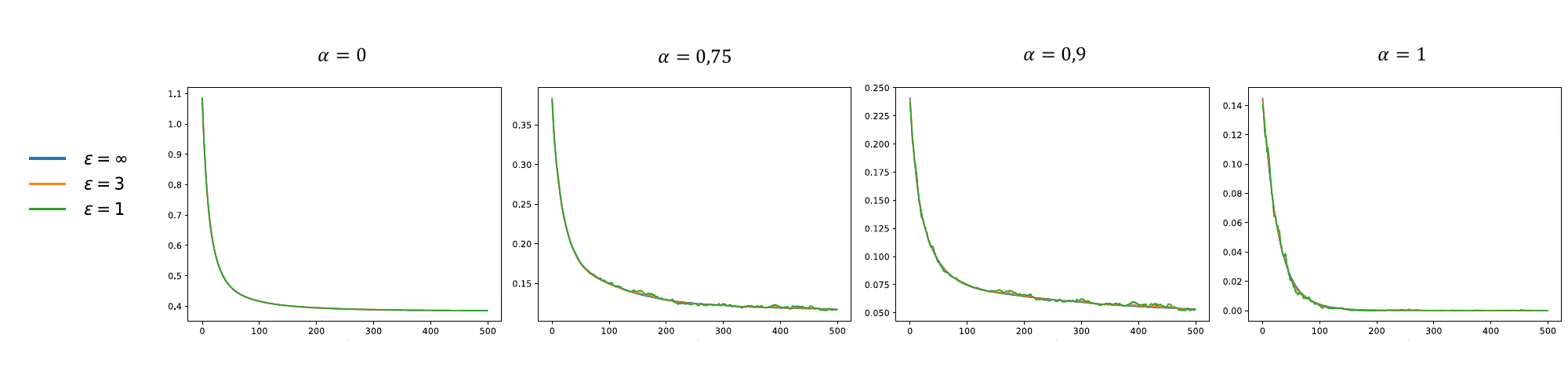}
    \caption{Training loss curve for the experiment of Figure \ref{fig:eoo}. Each graph represents the training loss \eqref{eq:loss_EOO} for a fixed value of $\alpha$ along the iterations of DP-SGD, for the different values of $\varepsilon$. }
    \label{fig:eoo_losses}
\end{figure}

\subsection{Regression}
 In our generating mechanism, the label $Y\in \{0,1\}$ is defined as a  set indicator function of a continuous response $Y^C \in [0,1]\times[0,1]$. From the data-generating process, it is easy to derive the distribution of $Y_C$ conditioned on the sensitive attribute. If $T_0$ denotes the triangle with vertices $(0,0),(0,1),(1,0)$ and $T_1$ the triangles with vertices $(0,1),(1,1),(1,0)$, then we know that $Y^C|A=j$ follows a mixture of the uniform distributions on $T_0$ and $T_1$, with weight $p$ in $T_0$ and $(1-p)$ in $T_1$ if $A=0$, and vice versa if $A=1$. The aim of this experiment is to perform  private and bi-dimensional fair regression over $Y_C$, which has never been considered before in the literature. To simplify our clipping bounds, we have centered our data to obtain a distribution in $[-1/2,1/2]\times[-1/2,1/2]$, and we have trained a two-layer neural network with hidden dimension 64, sigmoid activation function in the last layer,  with the output centered by subtracting  $(1/2,1/2)$, and minimizing the loss \eqref{eq:loss_SP}, with mean square loss $\ell(g_\theta(x),y^C) = \|g_\theta(x)- y^C\|_2^2$ and the sliced Wasserstein penalization applying to the output layer  $\varphi_\theta = g_\theta$. 

Figure \ref{fig:regression} shows the results of this experiment for different values of $\alpha$ and the privacy budget $\varepsilon$, with fixed $\delta = 0.1/n$, number of iterations $T=1000$, clipping constants $C=10$, $M=1/\sqrt{2}$, $L=\sqrt{2}$, learning rate $=0.05$ and number of projections in the Monte Carlo approximation $= 50$. From the visual inspection of the plots, we can appreciate that our statistical parity penalization helps to reduce the differences between the distributions of the predicted values. To aid visual inspection, we provide the values of the number of points over the diagonal for each class $A=0$ and $A=1$. If $g_\theta(x) = (g^1_\theta(x),g^2_\theta(x))$

$$ OD_0 = \frac{ \# \{X: g^2_\theta(X)>-g^1_\theta(X), A = 0\}}{n_0}$$
$$ OD_1 = \frac{ \# \{X: g^2_\theta(X)>-g^1_\theta(X), A = 1\}}{n_1}$$
 
Finally, \cref{fig:regression_losses} shows the convergence of the loss curve for the different values of $\alpha$ and $\varepsilon$. As in the previous examples, the private loss curves are simply noisy versions of the non-private ones.

\begin{figure}[h!]
    \centering
\includegraphics[width=1\linewidth]{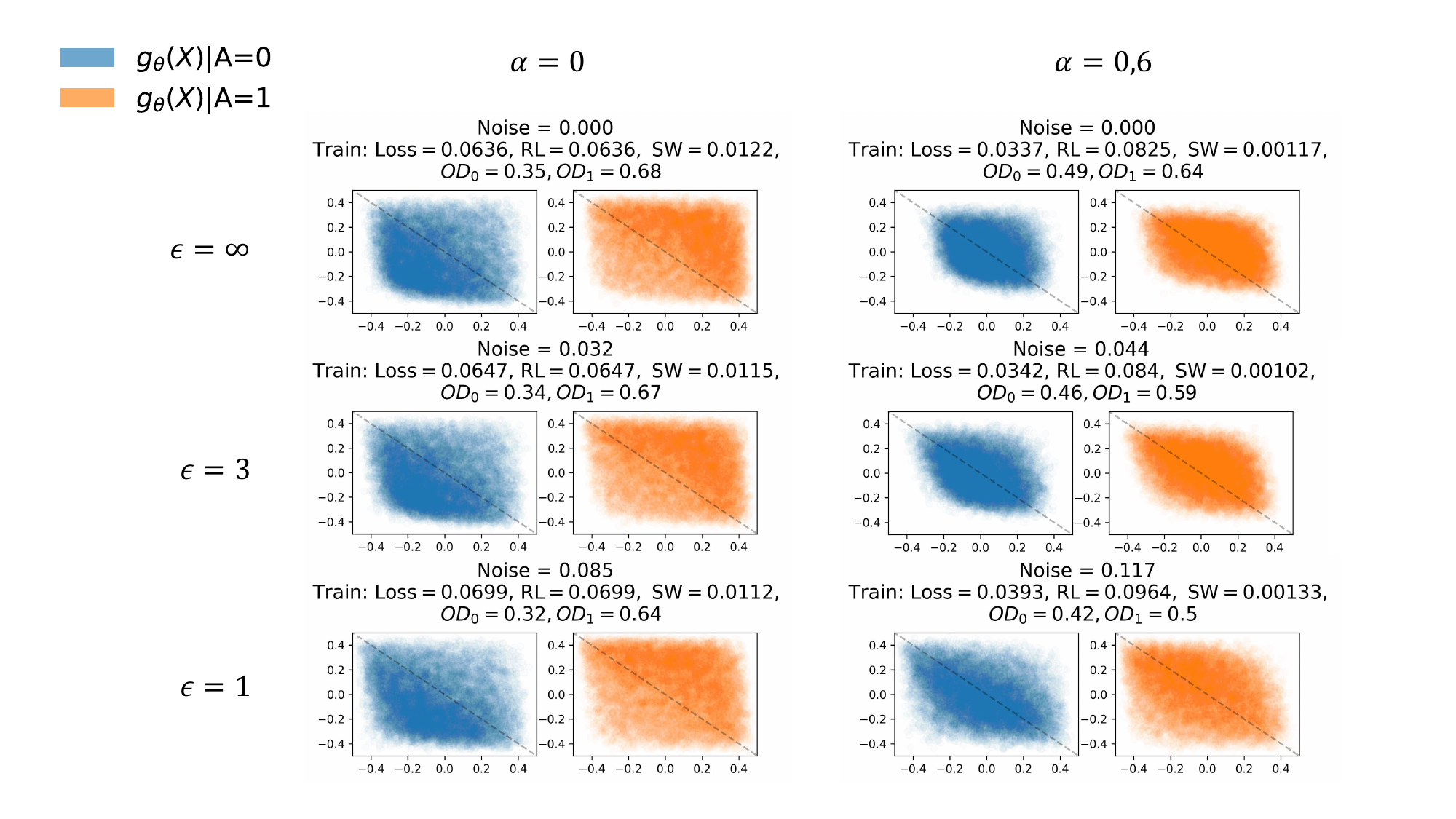}
    \caption{
    Predicted values $g_\theta(X)$ conditioned on the values of $A$ in the private regression model with SP regularization, for the different values of $\alpha$ and $\varepsilon$, and fixed $\delta = 0.1/n$. 
    Above each graph we indicate the noise added at each step of DP-SGD to obtain the desired privacy level, the value of the loss \eqref{eq:loss_SP} in the training procedure, together with the individual value of the regression loss (RL) and the  sliced Wasserstein loss (SW). Last line includes the indexes $OD_0$ and $OD_1$.}
    \label{fig:regression}
\end{figure}

\begin{figure}[h!]
    \centering
    \includegraphics[width=0.6\linewidth]{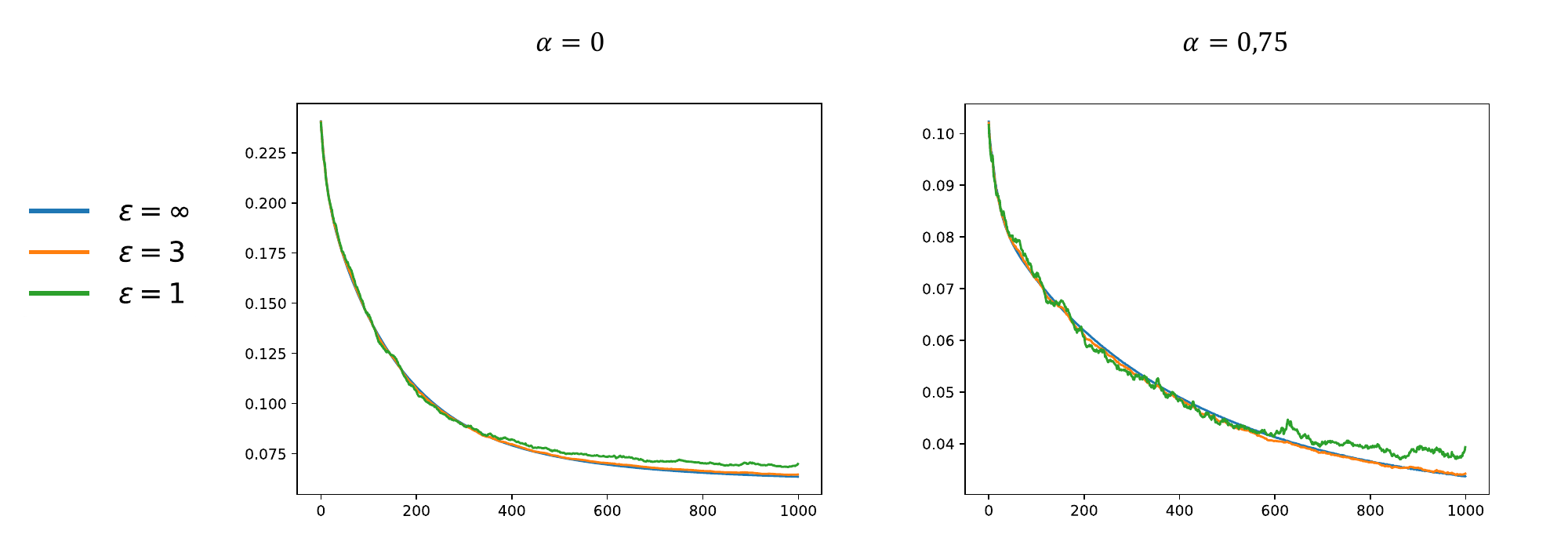}
    \caption{Training loss curve for the experiment of Figure \ref{fig:regression}. Each graph represents the training loss for a fixed value of $\alpha$ in \eqref{eq:loss_SP} along the iterations of DP-SGD, for the different values of $\varepsilon$. }
    \label{fig:regression_losses}
\end{figure}

\subsection{Representation learning.} Finally, we present another completely novel application of our procedure: fair representation learning. Using the same data as before, the objective is to privately learn an encoder $\varphi_{\theta_a}$ and a decoder $\psi_{\theta_b}$ minimizing the reconstruction mean squared error of the reconstructed values, penalized with the sliced Wasserstein distance to alleviate statistical parity unfairness present in the data. As in Section \ref{section:swae}, we denote $\theta=(\theta_a,\theta_b)$, $\varphi_{\theta}= \varphi_{\theta_a}$ and $\psi_\theta = \psi_{\theta_b}$. Now, the encoder and decoder are defined as fully connected neural networks with two layers, hidden dimension 62 and bi-dimensional latent space. Defining $g_\theta = \psi_\theta \circ \varphi_\theta$, we consider the reconstruction error $\ell(g_\theta(x),x) = \|g_\theta(x)- x\|_2^2$ and the fairness penalty is now imposed over the encoded representations given by $\varphi_\theta$.

As usual, \cref{fig:autoencoder} presents the result of training this model for different values of $\alpha$ and $\varepsilon$, with fixed $\delta=0.1/n$, number of iterations $T=500$ iterations, clipping values $C=10,M=2,L=\sqrt{2}$, learning rate $=0.01$ and number of projections in the Monte Carlo approximation $= 50$.
To assess the quality of the models, we have computed different comparative measures on an independent test sample. First, $RL_c$ denotes the reconstruction loss in the core part $X_{core}$, i.e. the first eight variables of $X$. The rest of the variables $X_{sp}$ are just a noisy version of $A$. Thus, $RL_c$ provides a measure of the error in the reconstruction loss for the relevant part of the data, and Figure \ref{fig:autoencoder} shows that for increasing values of $\alpha$, even though the reconstruction loss increases significantly, the reconstruction associated with the core part is not affected much. The other measures computed on the test data are the accuracy and disparate impact of a simple logistic regression model trained on the encoded representation of a portion (60\%) of the test data and evaluated on the remaining (40\%). We observe that increasing values of $\alpha$ lead to values of the disparate impact index closer to 1, at the expense of a decrease in accuracy.  Finally, we can infer from Figures \ref{fig:autoencoder} and \ref{fig:autoencoder_losses} that privacy doesn't affect much to the results of the optimization.

\begin{figure}[h!]
    \centering
\includegraphics[width=1\linewidth]{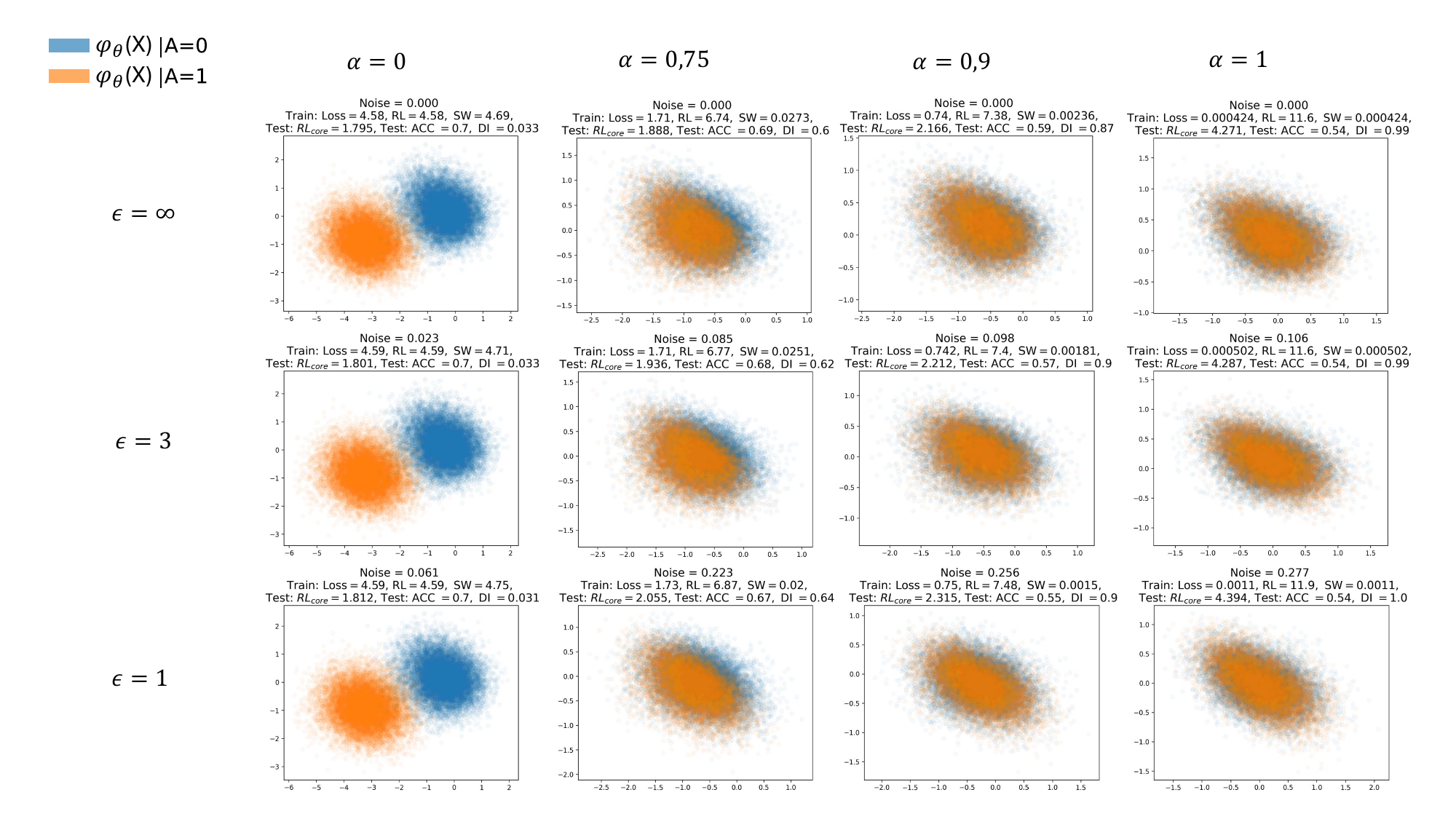}
    \caption{Plot of the latent space representations $\varphi_\theta(X)$ conditioned on the values sensitive attribute $A$ in the private representation learning model with SP regularization, for the different values of $\alpha$ and $\varepsilon$, and fixed $\delta = 0.1/n$. Above each plot we indicate the noise added at each step of DP-SGD to obtain the desired privacy level, the value of the loss \eqref{eq:loss_SP} in the training procedure, together with the individual value of the reconstruction loss (RL) and the sliced Wasserstein loss (SW). Last line includes the reconstruction loss on the core variables ($RL_1$), accuracy and disparate impact on test data.}
    \label{fig:autoencoder}
\end{figure}

\begin{figure}[h!]
    \centering
    \includegraphics[width=0.88\linewidth]{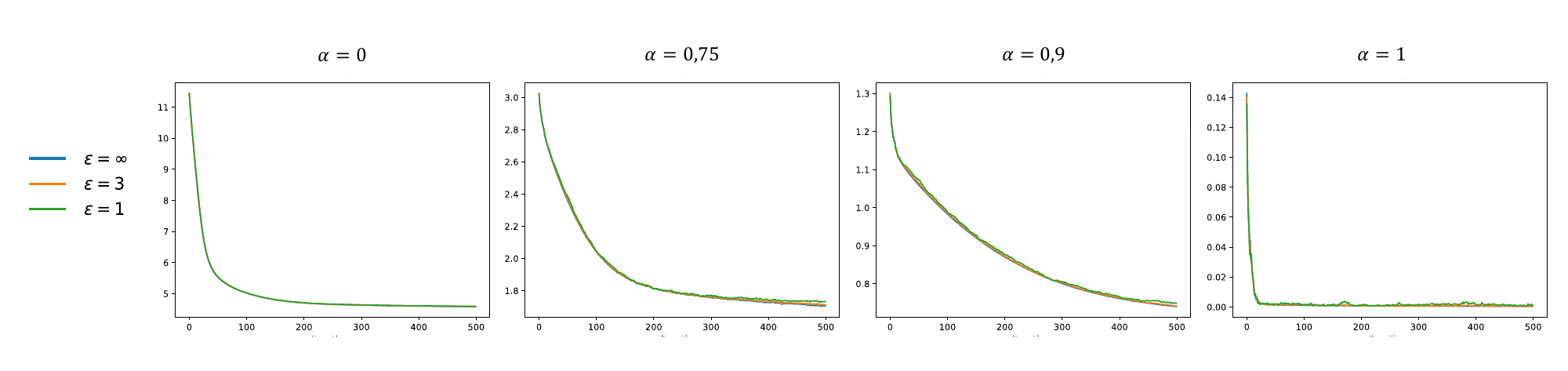}
    \caption{Training loss curve for the experiment of Figure \ref{fig:autoencoder}. Each graph represents the training loss \eqref{eq:loss_SP} for a fixed value of $\alpha$  along the iterations of DP-SGD, for the different values of $\varepsilon$. }
    \label{fig:autoencoder_losses}
\end{figure}

\subsection{Randomness strategy}

To provide a deeper understanding of the previous results, it is important to detail the role of randomness across the experiments. For each of the four bias mitigation experiments, the training process was repeated for different values of the weight $\alpha$ and privacy budgets $\varepsilon \in \{\infty, 3, 1\}$. To ensure a fair comparison across $(\alpha, \varepsilon)$ pairs, we used the same random seed for all configurations. As a result, neural networks share the same weight initialization, the sampled batches are identical at each iteration, and the standard Gaussian noise used in private updates is also the same. With this setup, our experiments isolate the influence of $\alpha$ and $\varepsilon$ on the training process. The only differences across configurations stem from:

\begin{itemize}
\item \textbf{Clipping:} Applied only when $\varepsilon < \infty$.
\item \textbf{Gradient composition:} The gradient is a linear combination of the classification gradient and the sliced Wasserstein gradient, weighted by $\alpha$.
\item \textbf{Noise multiplier:} The standard Gaussian noise is scaled by a factor that depends on both $\alpha$ (via sensitivity) and the required privacy budget $\varepsilon$.
\end{itemize}

To illustrate the effectiveness of this isolation strategy, we repeated the Statistical Parity classification experiment---chosen for its simplicity and clarity---for 10 different random seeds. For each seed, the model was trained across the 12 combinations of $(\alpha, \varepsilon)$ used in Figure~\ref{fig:di}.

The first row of Figure~\ref{fig:seeds} shows the average training loss for each $(\alpha, \varepsilon)$ pair, with $\pm 2\sigma$ confidence bands. The variability shows that if the randomness setup in Figures~\ref{fig:di} and~\ref{fig:di_losses} had not been fixed, significant fluctuations in the training loss across different values of $\varepsilon$ would have appeared, contrary to what is shown in Figure~\ref{fig:di_losses}. The isolation effect of our randomness setup is even more evident in the second row of Figure~\ref{fig:seeds}, which shows the differences between the private ($\varepsilon = 3, 1$) and non-private ($\varepsilon = \infty$) losses  computed using the same random seed, with the same $\pm 2\sigma$ confidence bands computed over the 10 seeds. These differences stay within a narrow range around zero, indicating that, in this classification task, adding privacy mainly introduces a small amount of noise around the non-private trajectory, supporting the conclusion obtained from Figure~\ref{fig:di_losses}.

\begin{figure}[ht]
    \centering
    \includegraphics[width=1\linewidth]{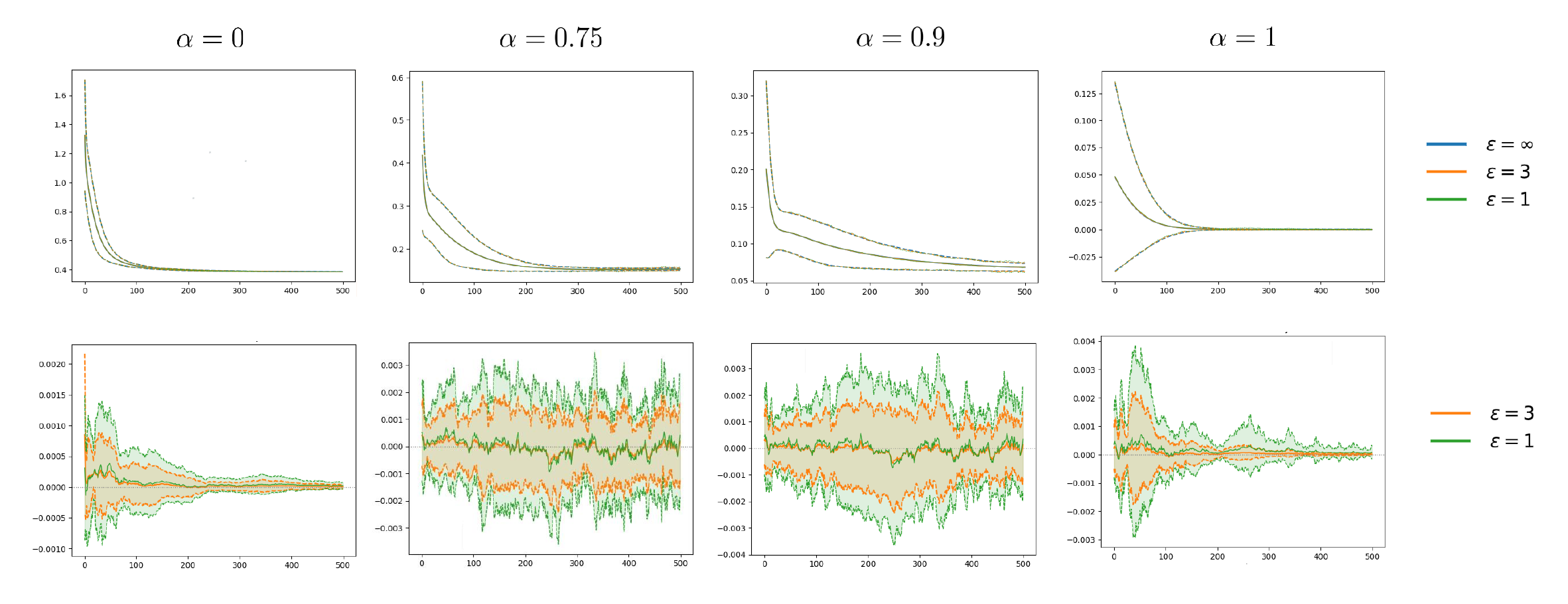}
\caption{\textbf{Top}: The solid line represents the average training loss \eqref{eq:loss_SP} over 10 random seeds for different values of $\varepsilon$ and $\alpha$. Dotted lines indicate the average $\pm 2$ standard deviations. For a given $\alpha$, the curves corresponding to different $\varepsilon$ values nearly overlap. \textbf{Bottom}: Differences between private ($\varepsilon = 3, 1$) and non-private ($\varepsilon = \infty$) losses, computed using the same random seed for each pair, and averaged over 10 different seeds. Shaded bands represent $\pm 2$ standard deviations.}

    \label{fig:seeds}
\end{figure}

\section{Private Distribution Matching and Data Generation} \label{section:distribution_matching}

A straightforward application of the sliced Wasserstein gradient is to learn a map mapping noise to a desired distribution.

\subsection{General method}

Given a private dataset $\vect{Z} \in \mathcal{Z}^n$ of target distribution $Z$, the goal is to learn, in a privacy-preserving manner, a map $g_{\theta}$ that minimizes
\begin{align*}
    \mathscr L(\theta) = SW_2^2\bigl(g_\theta \# P_{\vect X} , P_{\vect Z}\bigr)  \;,
\end{align*}
where $\vect{X} \in \mathcal{X}^n$ is simply a dataset of noise coming from a distribution from which sampling is easy of source distribution $X$.

\paragraph{Baseline of \cite{rakotomamonjy2021DPslicedWasserstein} with the fix of \cite{greenewald2024privacy}.}

In matrix notations, \cite{rakotomamonjy2021DPslicedWasserstein} releases $M(X)=X U+V$, where $X$ is the data matrix, $U$ is a random projection matrix (of which the directions will play the role of the projection directions in the sliced Wasserstein distance), and $V$ is Gaussian noise. The original privacy analysis of this mechanism is contested, as noted in Remark 1 of \cite{greenewald2024privacy}. In our experiments, we therefore use the fix to the privacy analysis provided in \cite{greenewald2024privacy} (Theorem 1). Note that from $M(X)$, it is then possible to compute the sliced Wasserstein distance, as well as its gradient w.r.t. the parameters of the network by post processing. For this baseline, we also incorporate standard privacy amplification by sub-sampling for fixed batch size without replacement for the replacement neighboring relation, as stated in \cite{DBLP:journals/corr/abs-2210-00597}. A dataset is sub-sampled at first, and we then use this sub-sampled dataset until the end of the training procedure. Note that by design, and contrary to our method, this approach is unable to provide privacy guarantees w.r.t. the source data. As described in Algorithm 1 in \cite{greenewald2024privacy}, the same amount of dummy privacy noise is added to the points in $g_\theta \# P_{\vect X}$ in order to learn the de-convolution map and not the map to the convoluted distribution.

\subsection{Results}

We considered $n = 100000$ samples of $Z$ drawn from the uniform distribution on a circle with radius $3/4$, and equal number of samples of $X$ drawn from the standard Gaussian distribution in $\mathbb{R}^2$. The function $g_\theta$ is defined as a fully connected neural network with an input dimension 2, three hidden layers with dimensions $(128, 64, 64)$, and an output dimension 2. Figure \ref{fig:gradient_flow} shows the evolution of the matching problem at different training steps. Thanks to \cref{theorem:gradient_sensitivity_sliced}, our methodology provides privacy guarantees for both the fixed variable $Z$ and the \textit{trained} variable $X$, in the sense that $g_\theta$ is applied to $X$. Above each plot, we can see the iteration number, the value of the loss, and the privacy budget $\varepsilon$ for both $X$ and $Z$, at each training step. The optimization parameters are $\delta = 0.1/n$,  batch size $ = 10,000$, learning rate $= 0.0075$, number of projections in the Monte Carlo approximation $= 50$, clipping values $M = 1$ and $L = 2\sqrt{2}$ (imposed using the suboptimal approach described in \cref{remark:degraded_bound}). To ensure more stable results, once we have privatized the gradient by adding noise, we clip the gradient again to improve the method's stability with the same clipping constants. 
Note that this step preserves privacy due to the post-processing property. 

The results for the method proposed in \cite{rakotomamonjy2021DPslicedWasserstein} are shown in \Cref{fig:baseline_data_generation}. For a fair comparison, we retained the same models and hyperparameters as in our approach, with the exception of reducing the number of projection directions to 15, as this adjustment improved the baseline's performance (trading increased bias for reduced privacy-related variance). Despite this tuning, the baseline performs poorly on this task, as shown by the substantially higher privacy budgets required to achieve acceptable results. Since the original paper \cite{rakotomamonjy2021DPslicedWasserstein} did not include visual experiments of this kind, it remains unclear whether the poor performance is intrinsic to the method, a consequence of the correction proposed by \cite{greenewald2024privacy}, or a limitation of our implementation.

\begin{figure}[h!]
    \centering
\includegraphics[width=1\linewidth]{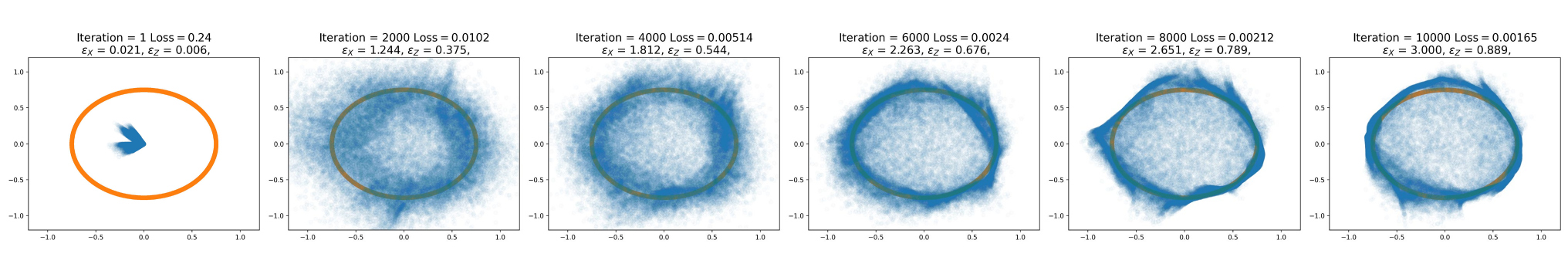} 
\caption{Data generation experiment. Samples from $X$ are represented in blue, samples from $Z$ in orange. Above each graph, we can see the iteration, the value of the loss, and privacy budgets w.r.t. the variables $X$ and $Z$.}
\label{fig:gradient_flow}
\end{figure}

\begin{figure}[h!]
    \centering
\includegraphics[width=0.30\linewidth]{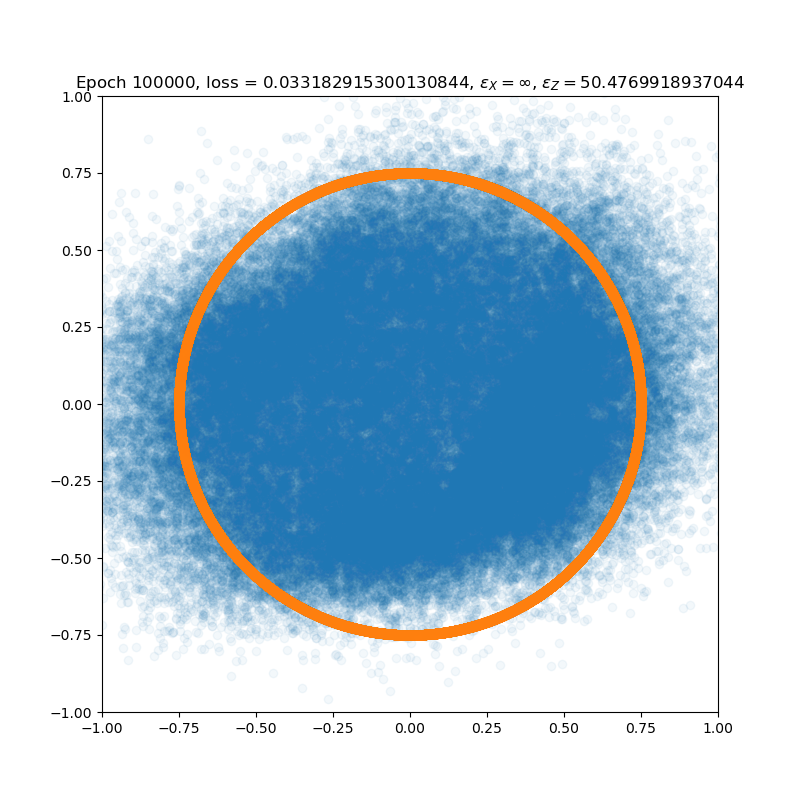} 
\includegraphics[width=0.30\linewidth]{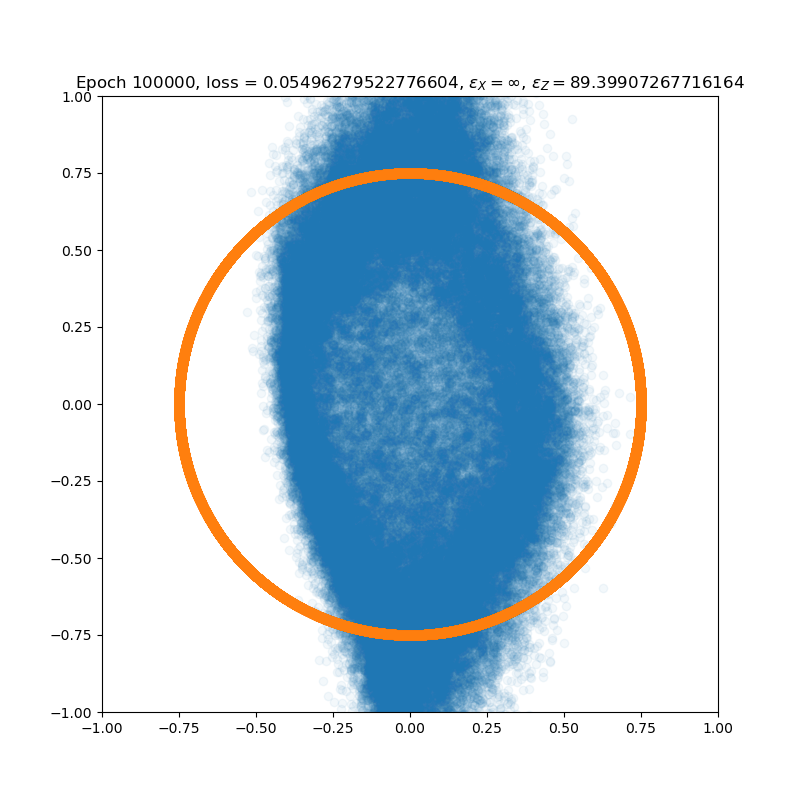} 
\includegraphics[width=0.30\linewidth]{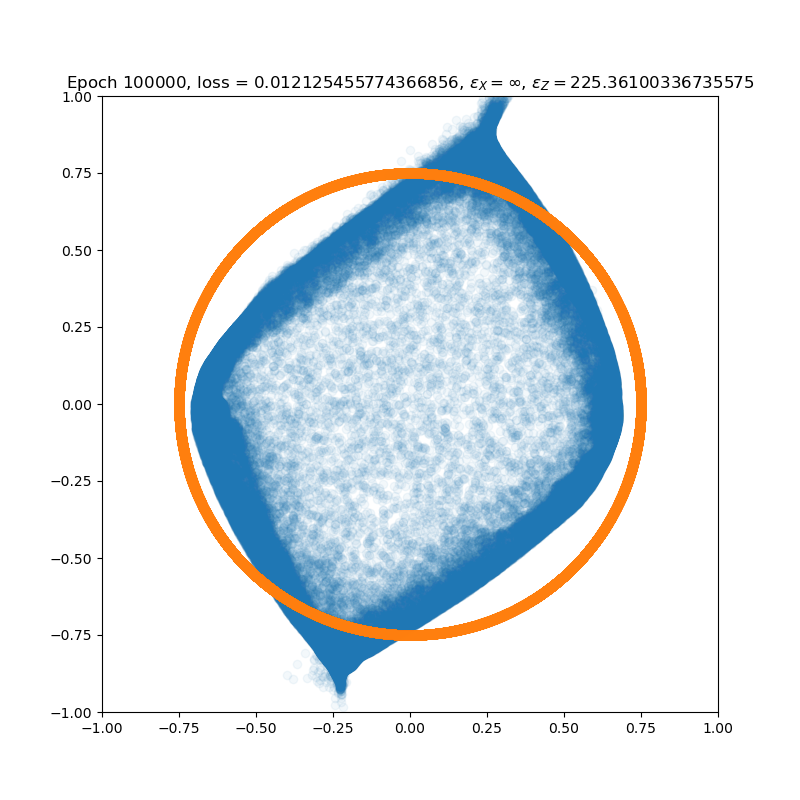} 
\caption{Baseline for the data generation experiment. Samples from $X$ are represented in blue, samples from $Z$ in orange. Above each graph, we can see the iteration, the value of the loss, and privacy budgets w.r.t. the variables $X$ and $Z$.}
\label{fig:baseline_data_generation}
\end{figure}

\section{Computational Details}
\label{sec:computational_details}

All experiments were conducted on a local server equipped with an NVIDIA RTX A6000 GPU with 48\,GB of memory. The MNIST and Fashion-MNIST experiments took 2.51 hours each. The four initial subproblems of fairness experiment required  1.18 hours in total for the 4 subproblems, whereas the randomness strategy required another 1.72 hours. The synthetic circles experiment took 1.64 hours. Thus, the total GPU time for the reported experiments was approximately 9.5 hours. The overall computational time was higher due to additional efforts related to code development, debugging, and the exploration of alternative approaches that were ultimately discarded. In total, we estimate the full computational effort to be approximately 210 hours.

The core of our implementation was developed in Python using the JAX library \cite{jax2018github}. The code represents a preliminary version aimed at illustrating the theoretical capabilities of our approach. We believe that a more refined implementation could lead to significantly faster execution. Notably, the current code does not leverage \texttt{jit}, the main just-in-time compilation tool in JAX, which could substantially reduce training time.

Nevertheless, the methodology presents inherent computational bottlenecks that limit the efficiency of our methodology. These bottlenecks are closely related to standard computational challenges in private learning. For illustration, assume the setting described in Remark~\ref{remark:comparison} with output dimension $ d = 1 $, and let $ r $ denote the number of parameters in $ \theta $. Consider batches $ \vect{X}_b $ and $ \vect{Z}_b $ of size $ m $. Using the same notation as in Remark~\ref{remark:comparison}, we aim to compute a private version of:

\begin{equation}\label{eq:matrix-product}
\underbrace{\nabla_\theta W_2^2(g_\theta(\vect{X}_b), \vect{Z}_b)}_{1 \times r} =
\underbrace{\nabla_{g_\theta} W_2^2(g_\theta(\vect{X}_b), \vect{Z}_b)}_{1 \times m} \times
\underbrace{\mathcal{J}_\theta g_\theta(\vect{X}_b)}_{m \times r}
\end{equation}

If privacy were not a concern, standard automatic differentiation would allow for efficient computation of the left-hand side, leveraging common operations and shared computational graphs. However, to enforce differential privacy, we need to adopt the inner clipping approximation described in Eq.~\eqref{eq:gradient_approx}. The expression \eqref{eq:gradient_approx} can be viewed as a product between a clipped version of $ \nabla_{g_\theta} W_2^2(g_\theta(\vect{X}_b), \vect{Z}_b) $ and a clipped version of $ \mathcal{J}_\theta g_\theta(\vect{X}_b) $. The main computational bottleneck arises from the need to compute the matrix $\mathcal{J}_\theta g_\theta(\vect{X}_b) \in \mathbb{R}^{m \times r}$. A similar challenge is usually encountered in differentially private optimization, where, if $g_\theta(x)$ represents the individual loss,  we need to compute $\frac{1}{m} \sum_{x \in \vect{X}_b} \text{clipp}_C \bigl(\nabla_\theta g_\theta(x)\bigr)$. Two key issues arise:

\begin{itemize}
    \item[(i)] Memory cost: If $ m $ and $ r $ are large, storing the full Jacobian $ \mathcal{J}_\theta g_\theta(\vect{X}_b) $ can become infeasible.  In standard differential privacy settings, this is typically addressed by reducing the batch size or by computing gradients across individual data points (or across smaller subbatches) sequentially and then aggregating them.

    \item[(ii)] Loss of GPU efficiency: These naive strategies undermine the advantages of batch computation on GPUs, leading to a significant increase in training time, as discussed in \cite{goodfellow2015efficient,lee2021scaling}. In particular, within the JAX framework, computing the full Jacobian via \texttt{jax.jacobian} is substantially more efficient than computing all the per-sample gradients with \texttt{jax.grad}.
\end{itemize}

In our setting, using small batch sizes $m$ is not suitable, as this would introduce substantial bias in the estimation of the Wasserstein distance. However, splitting the batch into smaller subbatches remains a viable strategy---if done carefully---as explained in Figure~\ref{fig:matrix-product}. The same discussion extends naturally to the general case of the sliced Wasserstein distance. If we flatten $g_\theta$---i.e., we see $g_\theta(\vect X_b) \in (\mathbb{R}^{d})^m$ as an element of $\mathbb{R}^{dm}$---then the chain rule yields the same decomposition as in Equation~\eqref{eq:matrix-product}, replacing $W_2$ with $SW_2$ (or its Monte Carlo approximation). In this setting, the dimension of the hidden dimension in the vector--matrix product becomes $dm$ instead of than $m$, further increasing the computational burden.

A further study of potential improvements to the naive strategy presented here is %
left for future research.

\begin{figure}[ht]
    \centering
    \includegraphics[width=1\linewidth]{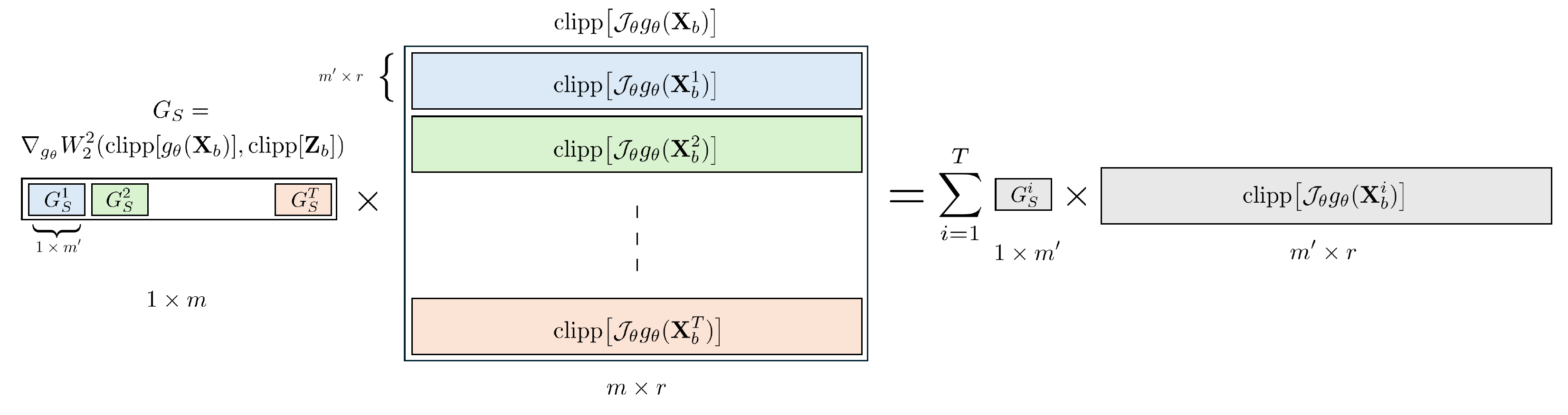}
    \caption{Naive approach used for computing the clipped gradient $ \nabla_\theta^{M,L} W_2^2(g_\theta(\vect{X}_b),\vect{Z}_b) $.  
    To avoid computing the full Jacobian matrix in the batch $ \mathcal{J}_\theta g_\theta(\vect{X}_b) $, we:  
    (i) split the batch $ \vect{X}_b = (\vect{X}_b^1, \ldots, \vect{X}_b^T) $,  
    (ii) Compute $G_S$, the gradient of the Wasserstein distance with clipping applied over the full batches $\vect X_b$ and $\vect Z_b$, as specified in Eq.~\eqref{eq:gradient_approx}.
    (iii) split $ G_S = (G_S^1, \ldots, G_S^T) $, and  
    (iv) iteratively compute $ \mathcal{J}_\theta g_\theta(\vect{X}_b^i) $, premultiply by $ G_S^i $, and aggregate.}
    \label{fig:matrix-product}
\end{figure}

\section{Counterexample for general cost functions} \label{counterexample}

In this section, we demonstrate that we cannot bound in general the sensitivity of the gradient if we use the Wasserstein loss function $W_p$, for general $p\geq 1$. Following the notation of Theorem \ref{theorem:gradient_sensitivity}, we denote by $\vect X = \{x_1,\ldots,x_n\} \subset \mathcal X^n$ the private dataset, $\vect Z=\{z_1,\ldots, z_n\} \subset \mathbb R$ the non-private dataset, and $P_{\vect X}, P_{\vect Z}$ the associated empirical distributions. Given $g_\theta: \mathcal X \rightarrow \mathbb R$ depending on the parameter $\theta$, and considering $h_\theta=I_d$, we study the sensitivity of $\Phi_\theta(\vect X)= \nabla_{\theta} W_p(g_\theta\# P_{\vect X}, P_{\vect Z})$, for the particular values of

\begin{itemize}
    \item $\vect X =\{x_1,\ldots,x_n\}$ with $x_i=\frac{i}{n}, \ i \in [n]$
     \item $\vect{\Tilde X} =\{\Tilde x_1,\ldots,\Tilde x_n\}$ with $\Tilde x_i=\frac{i-1}{n}, \ i \in [n]$
     \item $\vect Z=\{z_1,\ldots, z_n\}$, with $z_i = \frac{2i-1}{2n}, \ i\in [n].$
     \item $g_\theta(x) = x + \theta$
\end{itemize}

For this particular choice, $\vect X\sim_1\vect{\Tilde X}$, and  Assumption \ref{assumption1} and \ref{assumption2} in Theorem \ref{theorem:gradient_sensitivity} are verified for certain constants (once we restrict the domain of $\theta$). From the quantile representation, it is easy to compute 

\begin{align*}
     W_p(g_\theta\# P_{\vect X}, P_{\vect Z}) &= \Bigl(\frac{1}{n} \sum_{i=1}^n | x_i + \theta - z_i|^p \Bigr)^{1/p} =  \Bigl(\frac{1}{n} \sum_{i=1}^n \Big|\frac{1}{2n} + \theta \Big|^p \Bigr)^{1/p} \\ &= \left\{ \begin{array}{cc}
       \frac{1}{2n} + \theta   & \textnormal{if}\  \theta>-\frac{1}{2n} \\
        - \theta -\frac{1}{2n}   &  \textnormal{if}\ \theta\leq -\frac{1}{2n} 
     \end{array}\right. \\
    W_p(g_\theta\# P_{\vect{\Tilde X}}, P_{\vect Z}) &= \Bigl(\frac{1}{n} \sum_{i=1}^n | \Tilde x_i + \theta - z_i|^p \Bigr)^{1/p} =  \Bigl(\frac{1}{n} \sum_{i=1}^n \Big|-\frac{1}{2n} + \theta \Big|^p \Bigr)^{1/p}  \\ &= \left\{ \begin{array}{cc}
       -\frac{1}{2n} + \theta   & \textnormal{if}\  \theta>\frac{1}{2n} \\
         \theta -\frac{1}{2n}   &  \textnormal{if}\ \theta\leq \frac{1}{2n} 
     \end{array}\right.
\end{align*}

Therefore, by setting $\theta = 0$, we observe that the derivatives are $ \Phi_0(\vect X) = 1$ and $ \Phi_0(\vect{\Tilde X}) = -1$. Consequently, $ \Delta_2\Phi_\theta \geq 2$, indicating that the sensitivity does not decrease with the sample size $n$.

\section{Proofs}

\subsection{Proofs of \Cref{sec:WassersteinGradients}}

\begin{proof}[Proof of Proposition \ref{lemma:expression_W2}]

If we denote by $F$ and $G$ the distribution functions of $P_{\vect U}$ and $P_{\vect V}$, we know that 
\begin{align*}  W_2^2(P_{\vect U},P_{\vect V}) & = \int_0^1 (F^{-1}(t)-G^{-1}(t))^2dt \\
& = \int_0^1 \Bigl(\sum_{i=1}^{n} U_{(i)} I\Bigl(\frac{i-1}{n}<t\leq \frac{i}{n}\Bigr) -\sum_{j=1}^{m} V_{(j)} I\Bigl(\frac{j-1}{m}<t\leq \frac{j}{n}\Bigr) \Bigr)^2dt \\
& = \sum_{i=1}^{n} \sum_{j=1}^{m} (U_{(i)}-V_{(j)} )^2 \int_0^1  I\Bigl(\frac{i-1}{n}<t\leq \frac{i}{n},\frac{j-1}{m}<t\leq \frac{j}{n}\Bigr) dt \\
& = \sum_{i=1}^{n} \sum_{j=1}^{m} (U_{(i)}-V_{(j)} )^2   R_{i,j}  \\
&= \sum_{i=1}^{n} \sum_{j=1}^{m} (U_{(\sigma(i))}-V_{(\sigma(j))} )^2   R_{\sigma(i),\sigma(j)}\\
&= \sum_{i=1}^{n} \sum_{j=1}^{m} (U_{i}-V_{j} )^2   R_{\sigma(i),\sigma(j)} \ ,
\end{align*}
where the third equality follows from the fact that exactly one element in each sum is different from 0, and the fifth equality follows from reindexing the sum.\end{proof}
\subsection{Proofs of \Cref{sec:SensitivityPrivacy}}

\begin{proof}[Proof of Theorem \ref{theorem:gradient_sensitivity}]
First of all, note that (b) follows 
immediately from (a) and the definition of the neighboring relation $\sim_2$ in $\mathcal X^n\times \mathcal Z^m$. Consider two neighboring datasets $\vect X\sim \vect{\Tilde X}$ under the substitution  relation. We can assume without loss of generality that the datasets differ on the first observation $\Tilde x_1\neq x_1$. For ease of notation, denote $\vect{\Tilde{X}} = \{ \Tilde x_i\}_{i=1}^n$, even though $\Tilde x_i=x_i$ for $i\neq 1$. Along this proof, we will define $U_i:=g_\theta(x_i)$ and $\Tilde U_i := g_\theta(\Tilde x_i)$ for each  $i\in[n]$, and  $V_j:=h_\theta(z_j)$ for $j\in [m]$. Again, $U_i=\Tilde U_i$ for every $i\neq 1$. Define now the rank permutations $\sigma,\Tilde \sigma$ and $\tau$ such that
\begin{eqnarray*}
    U_i &= U_{(\sigma(i))}\ ,  \quad \ &i \in [n] \ ,\\
     \Tilde U_i &= \Tilde U_{(\Tilde \sigma(i))}\ ,  \quad \ &i \in [n] \ ,\\
      V_j &= V_{(\tau(j))}\ ,  \quad \ &j \in [m]\ .
\end{eqnarray*}
Denote $\vect U=(U_1,\ldots,U_n)$ and $\vect V=(V_1,\ldots,V_m)$. Corollary \ref{corollary:FormulaGradientsGeneral} ensures if we define  $P_{\vect{U}}= g_\theta \# P_{\vect{X}} =\frac{1}{n}\sum_{i=1}^n\delta_{U_i}$ and $P_{\vect V}= h_\theta \# P_{\vect Z}=\frac{1}{m}\sum_{i=1}^m \delta_{V_j}$ , then
\begin{align*}
    \nabla_{U,V} W_2^2(P_{\vect U}, P_{\vect V}) = \Biggl( \Bigl( 2\sum_{j=1}^m R_{\sigma(i),\tau(j)} (U_i-V_j)\Bigr)_{i\in[n]}, \Bigl( 2\sum_{i=1}^n R_{\sigma(i),\tau(j)} (V_j-U_i)\Bigr)_{j\in[m]}\Biggr) \ .
\end{align*}
Applying the chain rule, we obtain that 
\begin{align*}
    \nabla_\theta W_2^2(g_\theta\# P_{\vect{X}}, h_\theta \# P_{\vect{Z}})& = 2\sum_{i=1}^{n} \sum_{j=1}^{m} R_{\sigma(i),\tau(j)} (U_{i}-V_{j} ) \nabla_\theta g_\theta(x_i) \\
    & + 2\sum_{j=1}^{m}\sum_{i=1}^{n}  R_{\sigma(i),\tau(j)} (V_{j}-U_{i} ) \nabla_\theta h_\theta(z_j)
\end{align*}
Similarly, for the dataset $\vect{\Tilde X}$ we get
\begin{align*}    
    \nabla_\theta W_2^2(g_\theta\# P_{\vect{\Tilde X}}, h_\theta \# P_{\vect{Z}})& = 2\sum_{i=1}^{n} \sum_{j=1}^{m} R_{\Tilde \sigma(i),\tau(j)} (\Tilde U_i-V_{j} ) \nabla_\theta g_\theta(\Tilde x_i)\\
    &+ 2\sum_{j=1}^{m}\sum_{i=1}^{n}  R_{\Tilde \sigma(i),\tau(j)} (V_{j}-\Tilde U_i ) \nabla_\theta h_\theta(z_j)
\end{align*}

Therefore, 
\begin{align}
    \|&\nabla_\theta W_2^2(g_\theta\# P_{\vect{X}}, h_\theta \# P_{\vect{Z}})- \nabla_\theta W_2^2(g_\theta\# P_{\vect{\Tilde X}}, h_\theta \# P_{\vect{Z}}) \|_2 \leq \notag \\
    &\leq  2\ \Big\| \sum_{i=1}^{n} \sum_{j=1}^{m} R_{\sigma(i),\tau(j)} (U_{i}-V_{j} ) \nabla_\theta g_\theta(x_i) - \sum_{i=1}^{n} \sum_{j=1}^{m} R_{\Tilde \sigma(i),\tau(j)} (\Tilde U_i-V_{j} ) \nabla_\theta g_\theta(\Tilde x_i) \Big\|_2  \label{eq:term1}\\
    & +2\  \Big \| \sum_{j=1}^{m}\sum_{i=1}^{n}  R_{\sigma(i),\tau(j)} (V_{j}-U_{i} ) \nabla_\theta h_\theta(z_j) - \sum_{j=1}^{m}\sum_{i=1}^{n}  R_{\Tilde \sigma(i),\tau(j)} (V_{j}-\Tilde U_i ) \nabla_\theta h_\theta(z_j) \Big\|_2\label{eq:term2}
\end{align}

The term \eqref{eq:term2} is easier to bound, since the values inside $\nabla_\theta h_\theta(\cdot)$ coincide. First, note that 
\begin{align}\label{eq:propertyR}
    \sum_{j=1}^{m} R_{i,j} = \frac{1}{n} \ ,  \ \forall \ i\in [n] \quad \textnormal{and} \quad 
     \sum_{i=1}^{n} R_{i,j} = \frac{1}{m} \ ,\ \forall\  j\in [m] \ ,
\end{align}
The triangular inequality, the assumption $\|\nabla_\theta h_\theta(z)\|\leq L_2$ for every $z,\theta$ and the previous property allow us to derive the following inequalities
\begin{align}
    (\ref{eq:term2})  &=2\  \Big \| \sum_{j=1}^{m}V_{j} \nabla_\theta h_\theta(z_j)\Bigl( \sum_{i=1}^{n}  R_{\sigma(i),\tau(j)}-\sum_{i=1}^{n}  R_{\Tilde \sigma(i),\tau(j)}\Bigr)-\\
    & \hspace{3cm } - \sum_{j=1}^{m} \nabla_\theta h_\theta(z_j)\Bigl( \sum_{i=1}^{n}  U_iR_{\sigma(i),\tau(j)} -\sum_{i=1}^{n}  \Tilde U_i R_{\Tilde \sigma(i),\tau(j)}\Bigr)\Big\|_2 \notag\\
   &\leq 2 \sum_{j=1}^{m} \Bigl\|\nabla_\theta h_\theta(z_j)\Bigl( \sum_{i=1}^{n}  U_iR_{\sigma(i),\tau(j)} -\sum_{i=1}^{n}  \Tilde U_i R_{\Tilde \sigma(i),\tau(j)}\Bigr)\Big\|_2 \notag \\
   &\leq 2L_2\sum_{j=1}^{m} \Bigl| \sum_{i=1}^{n}  U_iR_{\sigma(i),\tau(j)} -\sum_{i=1}^{n}  \Tilde U_i R_{\Tilde \sigma(i),\tau(j)}\Big|\notag \\
   &=  2 L_2\sum_{j=1}^{m} \Bigl| \sum_{i=1}^{n}  U_{(i)}R_{i,\tau(j)} -\sum_{i=1}^{n}  \Tilde U_{(i)} R_{i,\tau(j)}\Big| \notag\\
   &= 2L_2 \sum_{j=1}^{m} \Bigl| \sum_{i=1}^{n}  R_{i,\tau(j)}(U_{(i)} - \Tilde U_{(i)})\Big|  \label{eq:last_line_ineq}
\end{align}
where the last lines follows from $U_i=U_{(\sigma(i))}$, $\Tilde U_i=\Tilde U_{(\Tilde \sigma(i))}$ and reindexing the sum. Since $U_i=\Tilde U_i$ for every $i\neq 1$, we know that
\begin{itemize}
    \item If $U_1\geq \Tilde U_1$, then $U_{(i)}\geq \Tilde U_{(i)}$ for every $i \in [n]$.
    \item If $U_1< \Tilde U_1$, then $U_{(i)}\leq \Tilde U_{(i)}$ for every $i\in[n]$.
\end{itemize}
This monotonicity property and the fact that $R_{i,j}\geq 0$ for every $i,j$ ensures that
\begin{align*}
    (\ref{eq:last_line_ineq})  &= 2L_2\ \Bigl| \sum_{j=1}^{m}  \sum_{i=1}^{n}  R_{i,\tau(j)}(U_{(i)} - \Tilde U_{(i)})\Big| \\
    &= 2L_2\ \Bigl|  \sum_{i=1}^{n} (U_{(i)} - \Tilde U_{(i)}) \sum_{j=1}^{m} 
 R_{i,\tau(j)}\Big| \\
 &=  \frac{2L_2}{n}\ \Bigl|  \sum_{i=1}^{n} (U_{(i)} - \Tilde U_{(i)}) \Big| \\
 &  = \frac{2L_2}{n} \Bigl|  \sum_{i=1}^{n} (U_{i} - \Tilde U_i) \Big| \\
 &= \frac{2L_2}{n} \bigr|U_{1} - \Tilde U_{1}\big| \\
 &\leq \frac{4L_2M}{n}
\end{align*}
By the triangular inequality, the term \eqref{eq:term1} can be bounded as follows
\begin{align}
   (\ref{eq:term1}) &\leq  2\  \Big \| \sum_{i=1}^{n} \nabla_\theta g_\theta(x_i)U_i\sum_{j=1}^{m} R_{\sigma(i),\tau(j)}   - \sum_{i=1}^{n} \nabla_\theta g_\theta(\Tilde x_i)\Tilde U_i\sum_{j=1}^{m} R_{\Tilde \sigma(i),\tau(j)} \Big\|_2 \label{eq:decomp1}\\
   & \quad +2\  \Big \| \nabla_\theta g_\theta(x_1)\sum_{j=1}^{m} R_{\sigma(1),\tau(j)} V_j  -  \nabla_\theta g_\theta(\Tilde x_1)\sum_{j=1}^{m} R_{\Tilde \sigma(1),\tau(j)}V_j \Big\|_2 \label{eq:decomp2} \\
   & \quad +  2\  \Big \| \sum_{i=2}^{n} \nabla_\theta g_\theta(x_i)\sum_{j=1}^{m} V_j (R_{\sigma(i),\tau(j)}   -  R_{\Tilde \sigma(i),\tau(j)} )\Big\|_2 \label{eq:decomp3}
\end{align}
We can bound independently each term in the decomposition, 
\begin{align}
    (\ref{eq:decomp1}) &= \frac{2}{n} \Big\| \sum_{i=1}^{n} \nabla_\theta g_\theta(x_i)U_i   -  \nabla_\theta g_\theta(\Tilde x_i)\Tilde U_i \Big\|_2 \notag \\
    &=\frac{2}{n}
 \Big\| \nabla_\theta g_\theta(x_1)U_1   - \nabla_\theta g_\theta(\Tilde x_1)\Tilde U_1 \Big\|_2 \notag \\
 &\leq \frac{2}{n}\Bigl(|U_1|\|\nabla_\theta g_\theta(x_1)\|_2 +|\Tilde U_1|\|\nabla_\theta g_\theta(\Tilde x_1)\|_2 \Bigr) \notag \\
 &\leq \frac{4L_1M}{n}\notag \\ 
 (\ref{eq:decomp2}) &\leq 2L_1M\Bigl(\sum_{j=1}^{m} R_{\sigma(1),\tau(j)}+\sum_{j=1}^{m} R_{\Tilde \sigma(1),\tau(j)}\Bigr)\notag \\
 &= \frac{4L_1M}{n}\notag \\
 (\ref{eq:decomp3})& \leq  2L_1\   \sum_{i=2}^{n} \Big |\sum_{j=1}^{m} V_j (R_{\sigma(i),\tau(j)}   -  R_{\Tilde \sigma(i),\tau(j)} )\Big| \notag\\
 &=2L_1\   \sum_{i=2}^{n} \Big |\sum_{j=1}^{m} V_{(j)} (R_{\sigma(i),j)}   -  R_{\Tilde \sigma(i),j} )\Big| \label{eq:last_inequality_2}
 \end{align}
 The last equality is a simple consequence of $V_j=V_{(\tau(j))}$ and reindexing the sum. To bound the last expression, it is useful to see that all the terms $\sum_{j=1}^{m} V_{(j)} (R_{\sigma(i),j)}   -  R_{\Tilde \sigma(i),j} )$ have the same sign, for $i=2,\ldots,n$. This will follow from the relationship between the permutations $\sigma$ and $\Tilde \sigma$. For instance, if $\Tilde \sigma(1)<\sigma(1)$, it follows that
 \begin{enumerate}
     \item[a)] \textit{$\Tilde \sigma(i)\geq \sigma(i)$ for every $i\geq 2$}.    
     Remember that $\sigma(i)$ denotes the position of $U_i$ in the ordered statistic $(U_{(1)},\ldots,U_{(n)})$, and $\Tilde \sigma(i)$ denotes the position of $\Tilde U_i$ in the ordered statistic $(\Tilde U_{(1)},\ldots,\Tilde U_{(n)}')$. Recall also that $\Tilde U_i=U_i$ for every $i\geq 2$. Therefore, $\Tilde \sigma(1)<\sigma(1)$ implies that $\Tilde U_1 < U_1$, and 
     \begin{itemize}
         \item If $\sigma(i)<\Tilde \sigma(1)$, then $ \Tilde \sigma(i)=\sigma(i)$.
         \item If $\sigma(i)=\Tilde \sigma(1)$, then $ \Tilde \sigma(i)=\sigma(i)$ if $U_i<\Tilde U_1$, and $ \Tilde \sigma(i)=\sigma(i)+1$ if $U_i>\Tilde U_1$.
          \item If $\Tilde \sigma(1)<\sigma(i)<\sigma(1)$, then $ \Tilde \sigma(i)=\sigma(i)+1$.
           \item If $\sigma(i)>\sigma(1)$, then $ \Tilde \sigma(i)=\sigma(i)$.
     \end{itemize}
     \item[b)] \textit{$\sum_{j=1}^{m} V_{(j)} (R_{\sigma(i),j)}   -  R_{\Tilde \sigma(i),j} ) \leq 0$ for every $i=2,\ldots,n$}. If we denote by $G$ the empirical distribution function of $V_1,\ldots,V_{m}$, then by definition of $R_{i,j}$,
     \begin{align*}
         \sum_{j=1}^{m} V_{(j)} &(R_{\sigma(i),j)}   -  R_{\Tilde \sigma(i),j} ) = \\
         &=\sum_{j=1}^{m} V_{(j)} \Biggl(\int_{\frac{\sigma(i)-1}{n}}^{\frac{\sigma(i)}{n}} I\Bigl( \frac{j-1}{m}<t\leq\frac{j}{m}\Bigr)dt-\int_{\frac{\Tilde \sigma(i)-1}{n}}^{\frac{\Tilde \sigma(i)}{n}} I\Bigl( \frac{j-1}{m}<t\leq\frac{j}{m}\Bigr)dt\Biggr)\\
         &= \int_{\frac{\sigma(i)-1}{n}}^{\frac{\sigma(i)}{n}}\sum_{j=1}^{m} V_{(j)}  I\Bigl( \frac{j-1}{m}<t\leq\frac{j}{m}\Bigr)dt-\int_{\frac{\Tilde \sigma(i)-1}{n}}^{\frac{\Tilde \sigma(i)}{n}}\sum_{j=1}^{m} V_{(j)}  I\Bigl( \frac{j-1}{m}<t\leq\frac{j}{m}\Bigr)dt\\
         &= \int_{\frac{\sigma(i)-1}{n}}^{\frac{\sigma(i)}{n}} G^{-1}(t)dt-\int_{\frac{\Tilde \sigma(i)-1}{n}}^{\frac{\Tilde \sigma(i)}{n}}G^{-1}(t)dt \\
         &= \int_{\frac{\sigma(i)-1}{n}}^{\frac{\sigma(i)}{n}} G^{-1}(t)-G^{-1}\Bigl(t+\frac{\Tilde \sigma(i)-\sigma(i)}{n}\Bigr)dt \leq 0
     \end{align*}
     for every $i=2,\ldots,n$, where the last bound is consequence of $(a)$ and the monotonicity of $G^{-1}$.
 \end{enumerate}
 Similarly, if $\Tilde \sigma(1)>\sigma(1)$, then $\Tilde \sigma(i)\leq\sigma(i)$ for every $i\geq 2$, which implies  $\sum_{j=1}^{m} V_{(j)} (R_{\sigma(i),j)}   -  R_{\Tilde \sigma(i),j} ) \geq 0$. Finally, the case $\Tilde \sigma(1)=\sigma(1)$ is trivial, since this implies $\Tilde \sigma=\sigma$. Therefore, in any of the cases, the sign property implies that 
 \begin{align*}
     (\ref{eq:last_inequality_2})& =  2L_1\  \Big | \sum_{i=2}^{n} \sum_{j=1}^{m} V_{(j)} (R_{\sigma(i),j)}   -  R_{\Tilde \sigma(i),j} )\Big| \\
     &= 2L_1\  \Big |  \sum_{j=1}^{m} V_{(j)} \sum_{i=2}^{n}(R_{\sigma(i),j)}   -  R_{\Tilde \sigma(i),j} )\Big| \\
     &= 2L_1\  \Big |  \sum_{j=1}^{m} V_{(j)} \sum_{i=1}^{n}(R_{\sigma(i),j)}   -  R_{\Tilde \sigma(i),j} ) - \sum_{j=1}^{m} V_{(j)} (R_{\sigma(1),j)}   -  R_{\Tilde \sigma(1),j} )\Big| \\
     &=  2L_1\  \Big |  \sum_{j=1}^{m} V_{(j)} (R_{\sigma(1),j)}   -  R_{\Tilde \sigma(1),j} )\Big| \\
     &\leq 2L_1M \Bigl( \sum_{j=1}^{m} R_{\sigma(1),\tau(j)}+\sum_{j=1}^{m} R_{\Tilde \sigma(1),\tau(j)}\Bigr) \\
     &= \frac{4L_1M}{n}
 \end{align*}
 Putting everything together, we can conclude that, 
\begin{align*}
    \|\nabla_\theta W_2^2&(g_\theta\# P_{\vect{X}}, h_\theta \# P_{\vect{Z}})- \nabla_\theta W_2^2(g_\theta\# P_{\vect{\Tilde X}}, h_\theta \# P_{\vect{Z}}) \|_2 \leq \frac{12L_1M}{n} + \frac{4L_2 M}{n} \ .
\end{align*}
\end{proof}

\subsection{Other Proofs}
\begin{proof}[Proof of Lemma~\ref{lemma:subsampling}]
    See the proof of Theorem 29 in \cite{DBLP:journals/corr/abs-2210-00597} which gives the result up to a minor adaptation. The term $\max \p{\frac{n_{1}'}{n_1}, \dots,  \frac{n_k'}{n_k}}$ indeed comes from considering the worst case analysis depending on which category the differing point is in.
\end{proof}

\begin{proof}[Proof of Corollary~\ref{th:fairandprivate}] 

Formally, with the notation of \Cref{definition:neighboring}, define for the first part $\mathcal D = \mathcal D_0^{n_0}\times \mathcal D_1^{n_1}$, where $\mathcal D_j = \mathcal X \times \mathcal Y \times \{j\}$ in the supervised case, and  $\mathcal D_j = \mathcal X  \times \{j\}$ in the unsupervised case, for $j=0,1$. Applying \cref{theorem:gradient_sensitivity_sliced} with $g_\theta = h_\theta$, we can bound the sensitivity of $\nabla_\theta \mathscr L^{SP}_\alpha(g_\theta)$ by \eqref{eq:sensitivity_SP}.
    
    For the second part, consider $\mathcal D =\prod_{j,k} \mathcal D_{j,k}^{n_{j,k}}$, where $\mathcal D_{j,k} = \mathcal X \times \{j\} \times \{k\}$, for $j \in \{0,1\}, k\in\{0,\ldots,R-1\}$. Under the relation $\sim_{2R}$, given two neighboring datasets, all the terms except one are the same in the sum in \eqref{eq:loss_SP}, and similarly for the gradient expression. More precisely,  under the assumptions of the theorem, with the notation adopted in \cref{section:fairness},
    \begin{align*}
        &\sup_{\vect D \sim_{2R} 
 \vect{\Tilde D}} \Big\|  \frac{1}{R} \sum_{k=1}^K \nabla_\theta{SW}_2^2\Bigl(  \varphi_\theta \# P_{\vect X_{0,k}}, \varphi_\theta \# P_{\vect X_{1,k}} \Bigr) -  \frac{1}{R} \sum_{k=1}^K \nabla_\theta{SW}_2^2\Bigl(  \varphi_\theta \# P_{\vect{\Tilde X}_{0,k}}, \varphi_\theta \# P_{\vect{\Tilde X}_{1,k}} \Bigr) \Big\|_2 \\
  &\leq \frac{1}{R}   \sup_{\vect D \sim_{2R} 
 \vect{\Tilde D}} \sum_{k=1}^K \Big\|  \nabla_\theta{SW}_2^2\Bigl(  \varphi_\theta \# P_{\vect X_{0,k}}, \varphi_\theta \# P_{\vect X_{1,k}} \Bigr) -   \nabla_\theta{SW}_2^2\Bigl(  \varphi_\theta \# P_{\vect{\Tilde X}_{0,k}}, \varphi_\theta \# P_{\vect{\Tilde X}_{1,k}} \Bigr) \Big\|_2 \\     
   &\leq \frac{1}{R}
   \sup_{\substack{k = 0, \ldots, R - 1 \\ (\vect X_{0,k},\vect X_{1,k}) \sim_{2} 
 (\vect{\Tilde X}_{0,k},\vect{\Tilde X}_{1,k})}} \hspace{-0.6cm} 
 \Big\|  \nabla_\theta{SW}_2^2\Bigl(  \varphi_\theta \# P_{\vect X_{0,k}}, \varphi_\theta \# P_{\vect X_{1,k}} \Bigr) -\nabla_\theta{SW}_2^2\Bigl(  \varphi_\theta \# P_{\vect{\Tilde X}_{0,k}}, \varphi_\theta \# P_{\vect{\Tilde X}_{1,k}} \Bigr) \Big\|_2 \\  
 &\leq \frac{1}{R}\max_{k=0\ldots,R-1} \frac{16ML}{\min \{n_{0,k},n_{1,k}\}} = \frac{1}{R}\frac{16ML}{\min_{j,k} \{n_{j,k}\}} 
    \end{align*}
which implies \eqref{eq:sensitivity_EO}.
    \end{proof}

\end{document}